\documentclass{article}
\usepackage{ijcai25}
\usepackage{times}
\usepackage{soul}
\usepackage{url}
\usepackage[hidelinks]{hyperref}
\usepackage[utf8]{inputenc}
\usepackage[small]{caption}
\usepackage{graphicx}
\usepackage{amsmath}
\usepackage{amsthm}
\usepackage{booktabs}
\usepackage{algorithm}
\usepackage{algpseudocode}
\usepackage[switch]{lineno}

\usepackage{amsthm}
\newtheorem{lemma}{Lemma}

\newtheorem{definition}{Definition}
\newtheorem{assumption}{Assumption}
\usepackage{amssymb}
\usepackage{amsmath}
\usepackage{bm}
\usepackage[american]{babel}
\def \x {\bm{x}}

\def \w {\bm{w}}

\def \D {\mathcal{D}}
\def \C {\mathcal{C}}
\def \B {\mathcal{B}}

\def \X  {\mathcal{X}}

\def \Y  {\mathcal{Y}}
\def \L  {\mathcal{L}}

\def \T {\mathcal{T}}

\def \G {\mathcal{G}}
\def \X {\mathcal{X}}
\def \Y {\mathcal{Y}}

\def \D {\mathcal{D}}

\newtheorem{theorem}{Theorem}

\title{On the Learning with  Augmented Class via Forests}

\author{
Fan Xu
\and
Wuyang Chen \And
Wei Gao \\
\affiliations
 National Key Laboratory for Novel Software Technology, Nanjing University, Nanjing, China\\
 School of Artificial Intelligence, Nanjing University, Nanjing, China\\
\emails
\{xuf, chenwy, gaow\}@lamda.nju.edu.cn
}

\begin{document}

\maketitle

\begin{abstract}
Decision trees and forests have achieved successes  in various real applications, most working with all testing classes known in training data. In this work, we focus on learning with augmented class via forests, where an augmented class may appear in testing data yet not in training data.
We incorporate information of augmented class into trees' splitting, that is, \textit{augmented Gini impurity}, a new splitting criterion is  introduced to exploit some unlabeled data from testing distribution.  We then develop the Learning with Augmented Class via Forests (short for LACForest) approach, which constructs shallow forests according to the augmented Gini impurity and then splits forests with pseudo-labeled augmented instances for better performance. We also develop deep neural forests via an optimization objective based on our augmented Gini impurity, which essentially utilizes the representation power of neural networks for forests. Theoretically, we present the convergence analysis for our augmented Gini impurity, and we finally conduct experiments to evaluate our approaches. The code is available at \url{https://github.com/nju-xuf/LACForest}.

\end{abstract}

\section{Introduction}

How to handle distribution changes has become an important problem  in a non-stationary learning environment \cite{Zhou2022}, and recent years have witnessed increasing attentions with various applications \cite{Gama:Zliobait:Blifet:Pechenizkiy:Bouchachia2014,Geng:Huang:Chen2021,Wang:Zhang:Su:Zhu2024}. This work focuses on learning with augmented class, that is, the class distribution changes and an augmented class unseen in the training data may emerge during the testing process \cite{Da:Yu:Zhou2014}. Here, we take object recognition in autonomous driving for example: new objects may appear on roads yet not in the historical labeled data, and  a reliable learning system should make good predictions over both known classes and augmented class.

Various approaches have been developed for learning with augmented class. \citeauthor{Da:Yu:Zhou2014}~\shortcite{Da:Yu:Zhou2014} studied  decision boundaries for augmented class under the low-density assumption. \citeauthor{Bendale:Boult2016}~\shortcite{Bendale:Boult2016} and \citeauthor{Rudd:Jain:Scheirer:Boult2017}~\shortcite{Rudd:Jain:Scheirer:Boult2017} estimated the probability of an instance belonging to augmented class on the basis of extreme value theory. \citeauthor{Mendes:Medeiros:Oliveira:Stein:Pazinato:Almeida2017}~\shortcite{Mendes:Medeiros:Oliveira:Stein:Pazinato:Almeida2017}
detected augmented class by a nearest neighbor method, and
\citeauthor{Liu:Garrepalli:Dietterich:Fern:Hendrycks2018}~\shortcite{Liu:Garrepalli:Dietterich:Fern:Hendrycks2018} presented PAC guarantees for the detection of augmented class. Some generative networks have also been applied to learn augmented class \cite{Ge:Demyanov:Garnavi2017,Neal:Olson:Fern:Wong:Li2018,Chen:Peng:Wang:Tian2021}. \citeauthor{Zhang:Zhao:Ma:Zhou2020}~\shortcite{Zhang:Zhao:Ma:Zhou2020} gave an unbiased risk estimation by exploiting unlabeled data, and \citeauthor{Shu:He:Wang:Wei:Xian:Feng2023}~\shortcite{Shu:He:Wang:Wei:Xian:Feng2023} generalized such approach to arbitrary loss functions.

Decision trees and forests have achieved great successes in various applications with strong generalization in handling discrete features and exploring local regions \cite{cutler:Edwards:Beard:Cutler:Hess:Gibson:Lawler2007,Qi2012,Grinsztajn:Oyallon:Varoquaux2022,Costa:Pedreira2023}. Most previous studies worked on the same distribution between training and testing data. For augmented class, \citeauthor{Mu:Ting:Zhou2017}~\shortcite{Mu:Ting:Zhou2017} and \citeauthor{Liu:Garrepalli:Dietterich:Fern:Hendrycks2018}~\shortcite{Liu:Garrepalli:Dietterich:Fern:Hendrycks2018} only studied its detection from labeled data of known classes, whereas it remains open on how to construct trees for learning with augmented class by exploiting unlabeled data from testing data.

This work aims to incorporate some useful information of augmented class into the construction of forests, and our main contributions can be summarized as follows:

\begin{itemize}
\item We introduce a new splitting criterion, i.e., \textit{augmented Gini impurity}, to incorporate information of augmented class from unlabeled data during the trees' splitting. We develop the \textit{LACForest} approach for learning with augmented class, which constructs shallow forests by augmented Gini impurity  and considers pseudo-labeled augmented instances for better performance.

\item For complex data with intrinsic structures (e.g., images), we develop deep neural forests for learning with augmented class, because of their powerful representations with an end-to-end training manner. We propose a new optimization objective for deep neural forests to learn augmented class, and the basic idea is to extend our augmented Gini impurity into a differentiable form by considering the mechanism of deep neural trees.

\item From a theoretical view, we present the convergence analysis of our augmented Gini impurity with respect to both decision trees and deep neural trees. We finally conduct extensive experiments to validate the effectiveness of our proposed approaches and perform some parameter influence analysis.
\end{itemize}

The rest of this work is organized as follows: Section~\ref{sec:pre} gives some preliminaries. Section~\ref{sec:LACForest} introduces augmented Gini impurity and LACForest approach. Section~\ref{sec:DeepLACForest} presents our deep neural approach.  Section~\ref{sec:exp} conducts some extensive experiments. Section~\ref{sec:con} concludes with future works.

\section{Preliminaries}\label{sec:pre}
Let $\X\subseteq \mathbb{R}^d$ and $\Y=\{1,\ldots,\kappa,\kappa+1\}$ be the instance and class spaces, respectively. Suppose that $\D$ is an underlying (unknown) distribution over $\X\times\Y$, and denote by $\D_\X$ its marginal distribution over $\X$.  We focus on learning with augmented class \cite{Da:Yu:Zhou2014}, where the first $\kappa$ classes can be observed in training data, while the $(\kappa+1)$-th class will merely emerge in testing data, which is known as an augmented class in training phase.

For known classes and augmented class, let $\D_\text{kc}$ and $\D_\text{ac}$ be the marginal distributions over space $\X\times\{1,\ldots,\kappa\}$ and $\X\times\{\kappa+1\}$ from distribution $\D$, respectively. We introduce the \emph{class shift assumption} as follows:
\begin{assumption}\label{def:class-shift}
We say that distribution $\D$ and its marginal distributions $\D_\text{kc}$ and $\D_\text{ac}$ satisfy class shift assumption if
\begin{equation} \label{eq:class-shift}
\D = (1-\theta) \D_\text{kc} + \theta \D_\text{ac}\ \text{ for some constant }\ \theta\in(0,1)\ .
\end{equation}
\end{assumption}

This assumption correlates data distributions for known classes and augmented class, which  has been well-studied for learning with augmented class \cite{Zhang:Zhao:Ma:Zhou2020,Shu:He:Wang:Wei:Xian:Feng2023} and open-set recognition \cite{Scheirer:Rocha:Sapkota:Boult2013}.

We focus on the semi-supervised setting for learning with  augmented class \cite{Da:Yu:Zhou2014,Liu:Garrepalli:Dietterich:Fern:Hendrycks2018,Zhang:Zhao:Ma:Zhou2020,Shu:He:Wang:Wei:Xian:Feng2023}, and the goal is to learn a function $f\colon\X\rightarrow \Y$ from labeled data $S_l$ and unlabeled data $S_u$ with
\[
S_l=\{(\x^l_1,y_1),\ldots,(\x^l_{n_l},y_{n_l})\} \text{ and }S_u=\{\x^u_1,\ldots,\x^u_{n_u}\}\ .
\]
Here, each labeled example $(\x^l_i,y_i)$  is drawn i.i.d. from $\D_\text{kc}$, and each unlabeled instance $\x^u_j$ is sampled i.i.d. from $\D_\X$.

Let $\mathbb{I}[\cdot]$ be the indicator function, which returns $1$ if the argument is true and $0$ otherwise. For integer $k>0$, we write $[k]=\{1,2,\cdots,k\}$, and denote by $\Delta_k$ the $k$-dimensional simplex.   For $z\in\mathbb{R}$, we denote by $\lfloor z\rfloor$ the largest integer no more than $z$, and define $(z)_+=\max(0,z)$.
Let $|A|$ be the cardinality of set $A$.

\section{Our LACForest Approach}\label{sec:LACForest}
This section proposes the \textit{LACForest} approach for learning with augmented class based on random forests, and the core idea is to introduce a new splitting criterion, \textit{augmented Gini impurity}, to incorporate potential information of augmented instances during the construction of decision trees.

\subsection{Augmented Gini Impurity}
For instance space $\C\subseteq \X$, unlabeled data $S_u$ and labeled data $S_l$, we introduce
\[
S_{\C,u}= S_u\cap \C \ \text{ and } \ S_{\C,l}= \{(\x,y)\colon (\x,y)\in S_l \text{ and } \x\in \C \}  .
\]
Denote by $n_{\C,l}=|S_{\C,l}|$ and $n_{\C,u}=|S_{\C,u}|$. We could define the \emph{augmented Gini impurity} as a splitting criterion.

\begin{definition}\label{def:ag-gini}
For an instance space $\C\subseteq\X$, we define the augmented Gini impurity $\G_{\C}(S_l,S_u)$ w.r.t. $S_l$ and $S_u$ as
\begin{equation} \label{eq:ag-gini}
\G_{\C}(S_l,S_u) = 1 - \sum_{k\in[\kappa+1]} \vartheta^2_{\C,k}(S_l,S_u)\ .
\end{equation}
Here, $\vartheta_{\C,\kappa+1}(S_l,S_u)$ is defined as
\[
\vartheta_{\C,\kappa+1}(S_l,S_u)=\mathbb{I}[n_{\C,u}>0]\Big(1-\frac{(1-\theta)n_u n_{\C,l} }{ n_l \max(1, n_{\C,u})}\Big)_+\ ,
\]
where $\theta$ is given in Eqn.~\eqref{eq:class-shift}; and for $k\in[\kappa]$, we define
\[
\vartheta_{\C,k}(S_l,S_u)=\sum_{(\x,y)\in S_l}(1-\vartheta_{\C,\kappa+1}(S_l,S_u))\frac{ \mathbb{I}[\x\in\C, y=k]}{\max(1,n_{\C,l})} \ .
\]
\end{definition}

\begin{figure}[t]
\centering
\includegraphics[width=3.2in]{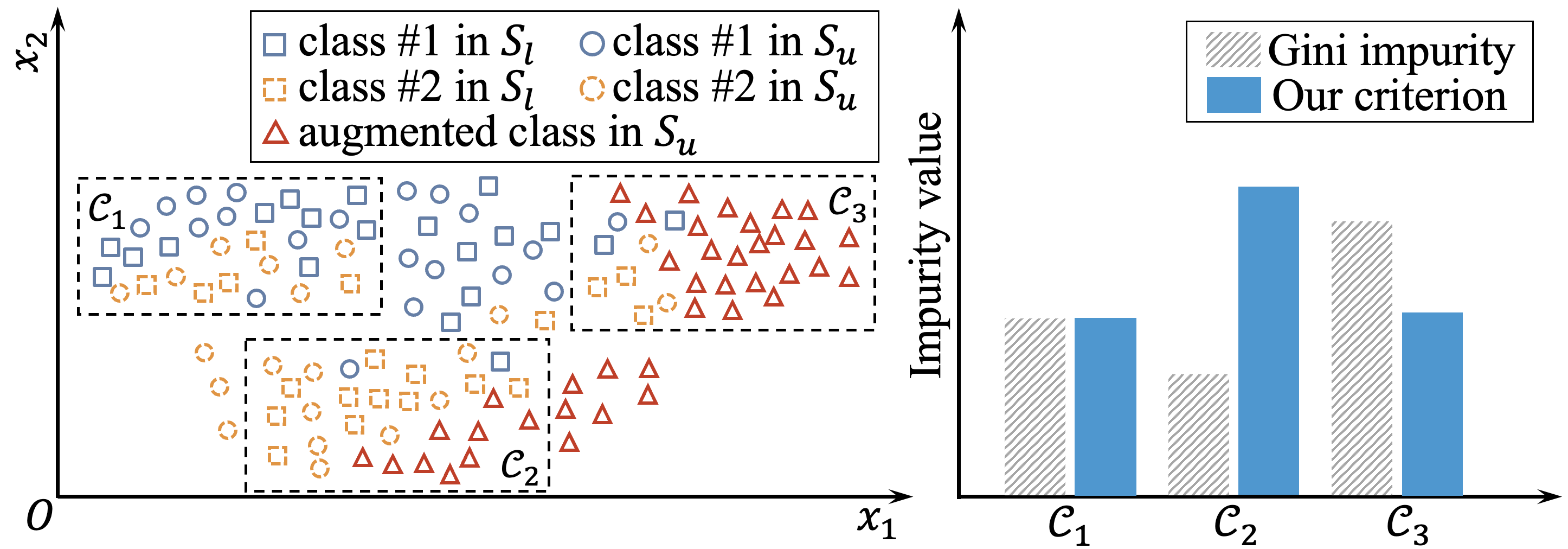}
\caption{An intuitive illustration on the difference between our criterion and original Gini impurity over a 2-dimensional dataset, by considering  augmented class.} \label{pic:comp}
\end{figure}

This definition essentially follows the original Gini impurity \cite{Breiman1984}, except with an additional term $\vartheta_{\C,\kappa+1}(S_l,S_u)$, which aims to incorporate some information of augmented class into impurity measure.

Figure~\ref{pic:comp} presents an intuitive illustration on the difference between our augmented and original Gini impurity over a 2-dimensional synthetic dataset. It is observable that our augmented Gini impurity could properly take augmented data into consideration compared with Gini impurity.

For $\vartheta_{\C,\kappa+1}(S_l,S_u)$, we also have
\begin{lemma}\label{lem:vartheta}
For $\delta\in(0,1)$ and instance space $\C\subseteq\X$, the following holds with probability at least $1-\delta$ over $S_l$ and $S_u$
\begin{eqnarray*}
\lefteqn{\left|\Pr_{(\x,y)\in\D}[y=\kappa+1|\x\in \C] - \vartheta_{\C,\kappa+1}(S_l,S_u) \right|} \\
&\leq&  O\left(\frac{\sqrt{\ln(1/\delta)}}{\gamma} \left(\frac{1}{\sqrt{n_l}}+\frac{\sqrt{\gamma}+1}{\sqrt{n_u}}\right) \right)\ ,
\end{eqnarray*}
if $n_{\C,u}/n_u\geq \gamma$ for some constant $\gamma\in(0,1)$.
\end{lemma}

Let $\L^*_\C$ be the optimal squared loss over distribution $\D$ under the condition $\x\in\C$ as follows:
\begin{equation} \label{eq:r*}
\L^*_\C = \min_{\w\in\Delta_{\kappa+1}} E_{(\x,y)\sim\D}\left[\| \w-\tilde{y}\|_2^2|\x\in \C\right] \ ,
\end{equation}
where $\tilde{\bm{y}}\in \mathbb{R}^{\kappa+1}$ is the one-hot encoding of $y\in[\kappa+1]$. Intuitively, $\L^*_\C$ shows the minimum expected squared loss  under the same prediction for all instances in space $\C$. 

Based on Lemma~\ref{lem:vartheta} and Assumption~\ref{def:class-shift}, we have
\begin{theorem} \label{thm:ag-gini}
For $\delta\in(0,1)$ and instance space $\C\subseteq\X$, the following holds with probability at least $1-\delta$ over $S_l$ and $S_u$
\[
\left|\L^*_{\C} - \G_\C(S_l,S_u)\right|\leq  O\left(\frac{\kappa\sqrt{\ln(\kappa/\delta)}}{\gamma/(1+\sqrt{\gamma})} \left(\frac{1}{\sqrt{n_l}}+\frac{1}{\sqrt{n_u}}\right) \right)
\]
if $n_{\C,u}/n_u\geq \gamma$ and $n_{\C,l}/n_l\geq \gamma$ for constant $\gamma\in(0,1)$.
\end{theorem}

This theorem shows that our augmented Gini impurity can be viewed as a good estimation of $\L^*_\C$ with the convergence rate $O(1/\sqrt{\min\{n_l,n_u\}})$. The detailed proof is given in Appendix~\ref{sec:proof:thm1}, which involves the equivalence between splitting criterions and loss functions, as well as the convergence of $\vartheta_{\C,\kappa+1}(S_l,S_u)$ in Lemma~\ref{lem:vartheta}.

\subsection{Construction of LACForest}
A tree model can be constructed by partitioning instance space into disjoint rectangular nodes recursively. Given node $\C\subseteq\X$, splitting feature $j\in[d]$ and splitting point $a\in\mathbb{R}$, we denote by the left and right children as
\[
\C^l_{j,a}=\{\x\in\C\colon x_j\leq a\} \ \text{ and } \
\C^r_{j,a}=\{\x\in\C\colon x_j>a \} \ ,
\]
respectively. We then solve the optimal splitting feature $j^*$ and splitting point $a^*$ from the following optimization
\begin{equation}\label{eq:split:prob}
(j^*,a^*) \in {\arg \max}_{j,a} \left\{R_{\C}(S_l,S_u, j,a) \right\} \ ,
\end{equation}
where the reduction of augmented Gini impurity is given by
\begin{multline*}
 R_{\C}(S_l,S_u, j,a)=\G_\C(S_{l},S_u) \\
- \frac{n_{\C^l_{j,a},u}}{n_{\C,u}} \G_{\C^l_{j,a}}(S_l,S_u) - \frac{n_{\C^r_{j,a},u}}{n_{\C,u}} \G_{\C^r_{j,a}}(S_l,S_u)\ .
\end{multline*}

From Theorem~\ref{thm:ag-gini}, we can distinguish augmented class from known classes in node $\C$ by solving Eqn.~\eqref{eq:split:prob}, as it converges to optimizing the squared loss as training data increases. From a theoretical view, we require sufficient training data in each node $\C$, i.e., $n_{\C,l}\geq \gamma n_l$ and $n_{\C,u}\geq \gamma n_u$ for some $\gamma\in(0,1)$ in Theorem~\ref{thm:ag-gini}; therefore, we should construct relatively shallow tree models.

On the other hand, we should take deeper tree models for better performance, and it is well-known that better predictions could be obtained by deepening tree models sufficiently \cite{Quinlan1986,Hastie:Tibshirani:Friedman2009}. Hence, our algorithm includes two main steps as follows:

\subsubsection{Step-I: Exploration of Augmented Instances}

We preliminarily construct $m$ random trees $\T'_1,\cdots,\T'_m$ for the exploration of augmented instances from Eqn.~\eqref{eq:split:prob}. For simplicity, we present the detailed construction of tree $\T'_1$, and consider other trees similarly. We initialize $\T'_1$ with one root node of instance space $\X$, and recursively repeat a procedure for each leaf node as follows:
\begin{itemize}
\item Select a $\tau$-subset $\mathcal{S}$ from $d$ available features randomly without replacement;
\item Solve the optimal splitting feature $j^*\in\mathcal{S}$ and point $a^*$  according to Eqn.~\eqref{eq:split:prob} under the constraints
\begin{eqnarray}
\min(n_{\C^l_{j,a},l},n_{\C^r_{j,a},l}) &\geq& \gamma n_l  \ , \label{eq:split:eq1}\\
\min(n_{\C^l_{j,a},u},n_{\C^r_{j,a},u}) &\geq& \gamma n_u  \ .\label{eq:split:eq2}
\end{eqnarray}
This ensures sufficient instances in each splitting node from Theorem~\ref{thm:ag-gini};
\item Split current node into left and right children via the optimal splitting feature $j^*$ and position  $a^*$.
\end{itemize}
The above procedure is stopped if there is no feasible solution for Eqn.~\eqref{eq:split:prob} under constraints in Eqns.~\eqref{eq:split:eq1} and \eqref{eq:split:eq2}.

Given random tree $\T'_1$ with $t'$ leaves $\C'_1,\C'_2,\cdots,\C'_{t'}$, we present the probability of augmented class for $\x\in \X$ as
\[
\T'_{1,\kappa+1}(\x) = \sum_{j\in [t']}   \vartheta_{\C'_j,\kappa+1}(S_{l},S_u) \mathbb{I} [\x\in \C'_j]\ ,
\]
where   $\vartheta_{\C'_j,\kappa+1}(S_{l},S_u)$ is an unbiased estimator for the proportion of augmented instances in $\C'_j$, given by Definition~\ref{def:ag-gini}.

\subsubsection{Step-II: Improvement of Prediction Performance}

\begin{algorithm}[t]
\caption{Our LACForest approach} \label{alg:LACForest}
\textbf{Input:} Labeled data $S_l$, unlabeled data $S_u$, proportion $\theta$, hyper-parameters $m,\tau,\gamma$ 
\begin{algorithmic}[1]
\Statex  \hspace*{-\algorithmicindent} \% Step-I: Exploration of Augmented Instances.
\For{$i\in\{1,\cdots,m\}$}
\State \parbox[t]{\dimexpr\linewidth-\algorithmicindent}{Grow random tree $\T'_i$ based on $S_l$ and $S_u$ according to the reduction of  augmented Gini impurity}.
\EndFor
\Statex  \hspace*{-\algorithmicindent} \% Step-II: Improvement of Prediction Performance
\State Calculate $\mathcal{LF}'_{\kappa+1}(\x)$ for $\x\in S_u$ by Eqn.~\eqref{eq:avgAgscore}.
\State Label top $n_p=\lfloor \theta n_u \rfloor$  instances in $S_u$ of the highest $\mathcal{LF}'_{\kappa+1}(\x)$ and obtain $\tilde{S}$.
\For{$i\in\{1,\cdots,m\}$}
\State \parbox[t]{\dimexpr\linewidth-\algorithmicindent}{Split random tree $\T'_i$ based on $S_l\cup \tilde{S}$ according to the reduction of Gini impurity and obtain $\T_i$.}
\EndFor \\
\Return $\mathcal{LF}(\x) = \underset{k\in [\kappa+1]}{\arg\max} \sum_{i\in[m]} \T_{i,k}(\x)$.
\end{algorithmic}
\end{algorithm}

Given $m$ trees $\T'_1,\cdots,\T'_m$, we calculate the average augmented score of instance $\x\in\X$ by
\begin{equation}\label{eq:avgAgscore}
\mathcal{LF}'_{\kappa+1}(\x) = \frac{1}{m}\sum_{i\in[m]} \T'_{i,\kappa+1}(\x) \ .
\end{equation}
We then select $\lfloor \theta n_u \rfloor$ instances in $S_u$ of the highest average augmented scores, and label them with pseudo-labels of augmented class. Without loss of generality, we denote by
\[
\tilde{S} = \{(\x^u_1,\kappa+1),(\x^u_2,\kappa+1),\cdots,(\x^u_{\lfloor \theta n_u \rfloor},\kappa+1)\}\ .
\]
Intuitively, $\tilde{S}$ retains the most likely augmented instances for further partition, and similar approaches have been studied by \cite{Liu:Lee:Yu:Li2002,Tanha:Van:Afsarmanesh2017}.

We further partition $\T'_1,\cdots,\T'_m$ for better predictions based on $\tilde{S}\cup S_l$. We also repeat the following procedure recursively for each leaf $\C$ of $\T'_1$:
\begin{itemize}
\item Select a $\tau$-subset $\mathcal{S}$ from $d$ available features randomly without replacement;
\item Solve the optimal splitting feature $j^*\in\mathcal{S}$ and position $a^*$ w.r.t. Gini impurity and instances in $\tilde{S}\cup S_l$;
\item Split the current node into left and right children via the optimal splitting feature $j^*$ and position  $a^*$.
\end{itemize}
The above procedure is stopped if all instances have the same label in a leaf, and get the final random tree $\T_1$. Similarly, we repeat the above procedure for $\T'_2,\cdots,\T'_m$.

\medskip

Given $m$ random trees $\T_1,\cdots,\T_m$, we predict the final label of instance $\x\in\X$ as
\[
\mathcal{LF}(\x) = \underset{k\in [\kappa+1]}{\arg\max} \sum_{i\in[m]} \T_{i,k}(\x) \ ,
\]
where $\T_{i,k}(\x)$ is the probability of the $k$-th category of tree $\T_i$ with leaves  $\C_{i,1},\C_{i,2},\cdots,\C_{i,t_i}$, i.e.,
\begin{equation*}
\T_{i,k}(\x)=\sum_{j\in[t_i]} \frac{\mathbb{I}[\x\in \C_{i,j}] \sum_{(\x',y)\in S_l\cup \tilde{S}}\mathbb{I}[\x'\in \C_{i,j}, y=k]}{\max(1,\sum_{(\x',y)\in S_l\cup \tilde{S}}\mathbb{I}[\x'\in \C_{i,j}])} \ .
\end{equation*}

Algorithm~\ref{alg:LACForest} gives detailed descriptions of our  LACForest approach. Notice that our method requires the proportion $\theta$ of augmented class, which is usually unknown in practice. One feasible solution is to consider the kernel mean embedding method from \cite{Ramaswamy:Scott:Tewari2016}, and this is similar to previous studies on learning with augmented class \cite{Zhang:Zhao:Ma:Zhou2020,Shu:He:Wang:Wei:Xian:Feng2023}.

\section{Deep Neural LACForest} \label{sec:DeepLACForest}

This section explores deep neural forests via augmented Gini impurity for learning with augmented class, which presents powerful representations with an end-to-end training manner for complex data such as images and texts. For neural forests, another advantage is to take soft splits rather than hard splits of decision trees, which could yield smoother decision boundaries and avoid overfitting \cite{Kontschieder:Fiterau:Criminisi:Bulo2015}. In this work ,we take one step on learning with augmented class via deep neural forests,  while previous works mostly focused on all known classes \cite{Kontschieder:Fiterau:Criminisi:Bulo2015,Tanno:Arulkumaran:Alexander:Criminisi:Nori2019,Ji:Wen:Zhang:Du:Wu:Zhao:Liu:Huang2020,Li:Cai:Yan2024}.

\subsection{Augmented Gini Impurity for Neural Trees}
A deep neural forest is constructed with a DNN encoder $h$ and $m$ differentiable neural trees $\mathcal{D}\mathcal{T}_1,\mathcal{D}\mathcal{T}_2,\cdots,\mathcal{D}\mathcal{T}_m$. The key point is how to learn differentiable neural trees by incorporating information of augmented class, and our idea is to consider augmented Gini impurity as an optimization objective for neural trees.

We take neural tree $\mathcal{D}\mathcal{T}_1$ as an example, and it is essentially an $l$-layer complete binary tree with $l\geq2$. We denote by
$\{\B_1,\cdots,\B_t\}$ and $\{\B_{t+1},\cdots,\B_{2t+2}\}$ all internal and leaf nodes of  $\mathcal{D}\mathcal{T}_1$, respectively, where  $t=2^{l-1}-1$.

For every internal node $\B_{i}$ ($i\in[t]$), we associate with a function $f_i\colon\X\rightarrow [0,1]$ to represent the probability of being assigned to the left child, and hence the probability of the right child is $1-f_i(\x)$. Given encoder $h$, we define
\[
f_i(\x) = \text{sigmoid}(\w_{i}^T h(\x) + b_i) \ ,
\]
where sigmoid$(z)=(1+\exp(-z))^{-1}$ and $\w_i,b_i$ are learned parameters, as in \cite{Kontschieder:Fiterau:Criminisi:Bulo2015}.

\begin{algorithm}[t]
\caption{Deep Neural LACForest} \label{alg:DeepLACForest}
\textbf{Input:} Labeled data $S_l$, unlabeled data $S_u$, proportion $\theta$,
hyper-parameters $m,l,T,\lambda_{\text{ce}}$
\begin{algorithmic}[1]
\State Randomly initialize $h$ and $\mathcal{D}\mathcal{T}_1,\cdots,\mathcal{D}\mathcal{T}_m$.
\For {$j\in \{1,\cdots,T\}$}
\State Break $S_l$ and $S_u$ into random mini-batches.
\For{mini-batches $S'_l$ and $S'_u$ from $S_l$ and $S_u$}
\State Compute $\L$ by Eqn.~\eqref{eq:loss} and compute gradients.
\State Update $h$ and $\mathcal{D}\mathcal{T}_1,\cdots,\mathcal{D}\mathcal{T}_m$ by SGD.
\EndFor
\EndFor \\
\Return $\mathcal{DF}(\x) =  \underset{k\in [\kappa+1]}{\arg\max} \sum_{i\in[m]} \mathcal{D}\mathcal{T}_{i,k}(\x)$.
\end{algorithmic}
\end{algorithm}

For every leaf node $\B_i$ $(i\in[2t+2]\setminus [t])$, let $\mathcal{I}(\x\to\B_i)$ be the event that $\x$ is assigned to leaf $\B_i$, and denote by
\begin{eqnarray*}
\mu_{\B_i}(\x)&=&\Pr\left[\mathcal{I}(\x\to\B_i)\right]\\
 &=&\prod_{j\in[t]} f_j(\x)^{\mathbb{I}[\B_i\swarrow\B_{j}]} (1-f_{j}(\x))^{\mathbb{I}[\mathcal{B}_j \searrow\B_{i}]} \ ,
\end{eqnarray*}
where $\B_i\swarrow \B_{j}$ and $\mathcal{B}_j \searrow\B_{i}$ represent that $\B_i$ belongs to the left and right subtrees of $\B_{j}$, respectively. We also denote by
\[
n_{\B_i,u}=\sum_{\x\in S_u} \mu_{\B_i}(\x)\quad \text{and} \quad n_{\B_i,l}=\sum_{(\x,y)\in S_l} \mu_{\B_i}(\x) \ ,
\]
for unlabeled data $S_u$ and labeled data $S_l$, respectively.

\begin{definition} \label{def:ag-gini:B}
For $i\in[2t+2]\setminus[t]$, we define the augmented Gini impurity  w.r.t. $S_l$ and $S_u$ as
\begin{equation} \label{eq:ag-gini:B}
\G_{\B_i}(S_l,S_u) = 1 - \sum_{k\in[\kappa+1]} \vartheta^2_{\B_i,k}(S_l,S_u)\ ,
\end{equation}
where $\vartheta_{\B_i,\kappa+1}(S_l,S_u)$ is defined as
\begin{equation*}
\vartheta_{\B_i,\kappa+1}(S_l,S_u)=\left(1-\frac{ (1-\theta)n_u n_{\B_i,l}} {n_l n_{\B_i,u} }\right)_+ \ ,
\end{equation*}
and for $k\in[\kappa]$, $\vartheta_{\B_i,k}(S_l,S_u)$ is given by
\begin{equation*}
(1-\vartheta_{\B_i,\kappa+1}(S_l,S_u))\sum_{(\x,y)\in S_l}\frac{ \mu_{\B_i}(\x)}{n_{\B_i,l} }\mathbb{I}[y=k] \ .
\end{equation*}
\end{definition}

\begin{table*}[t]
\centering
\footnotesize
\resizebox{1\linewidth}{!}{
\renewcommand{\arraystretch}{1.1}
\begin{tabular}{|cccc|cccc|cccc|cccc|}
\hline
Datasets  & \#inst & \#feat & \#class &  Datasets &  \#inst&  \#feat &  \#class &  Datasets  & \#inst & \#feat & \#class &  Datasets &  \#inst&  \#feat &  \#class  \\
\hline
\textsf{segment}  &  2,310 & 19  & 7   &  \textsf{mfcc}     &   7,195 &  22   &   4  &  \textsf{drybean}  &  13,661  &  16  &  7
&  \textsf{mnist}    &  70,000  &  784 &  10  \\
\textsf{texture}  &  5,500 & 40  & 11  &  \textsf{usps}     &  9,298   & 256  &  10  &  \textsf{letter}    &  20,000  &  16  &  26  &  \textsf{fmnist}     &  70,000  &  784 &  10   \\
\textsf{optdigits}&  5,620 & 64  & 10  &  \textsf{har}      &  10,299  & 562  &  6   &  \textsf{shuttle}     &  58,000  &  9  &  7  &  \textsf{kuzushiji} &  70,000  &  784 &  10 \\
\textsf{satimage} &  6,435 & 36  & 6   &  \textsf{mapping}  &  10,845  & 28   &  6   &  \textsf{drive}     &  58,509  &  48  &  11  &  \textsf{svhn}     & 99,289 &  3072 &  10 \\
\textsf{landset}  &  6,435 & 73  & 6   &  \textsf{pendigits}&  10,992  &  16  &  10 &  \textsf{senseveh}  &  61,581  &  100 &  3 &  \textsf{cifar10}  & 60,000 &  3072 &  10 \\
\hline
\end{tabular}}
\caption{Details of datasets.}
\label{tab:benchmark}
\end{table*}

This definition follows the augmented Gini impurity of decision trees in  Definition~\ref{def:ag-gini}, whereas deep neural trees consider probabilistic assignment of instances to nodes while the latter focuses on deterministic assignment. Essentially, $\vartheta_{\B_i,k}(S_l,S_u)$ estimates the probability of the $k$-th class under the condition $\mathcal{I}(\x\to\B_i)$. Based on Assumption~\ref{def:class-shift}, we have
\begin{lemma} \label{lem:thm:ag-gini:b}
For $k\in[\kappa+1]$, $\delta\in(0,1)$ and $i\in[2t+2]\setminus[t]$, we have, with probability at least $1-\delta$ over $S_l$ and $S_u$,
\begin{eqnarray*}
\lefteqn{\left|\Pr_{(\x,y)\sim\D}[y=k|\mathcal{I}(\x\to\B_i)] - \vartheta_{\B_i,k}(S_l,S_u)\right|}\\
&\leq& O\left(\frac{\sqrt{\ln(1/\delta)}}{\gamma^2/(\gamma+1)}\left(\frac{1}{\sqrt{n_l}}+\frac{1}{\sqrt{n_u}}\right)\right) \ ,
\end{eqnarray*}
if  ${\mathbb{E}}_{\D}[\mu_{\B_i}(\x)]\geq \gamma$, $\mathbb{E}_{\D_\text{kc}}[\mu_{\B_i}(\x)]\geq \gamma$, $n_{\B_i,u}\geq \gamma n_u$ and $n_{\B_i,l}\geq \gamma n_l$ for some constant $\gamma\in(0,1)$.
\end{lemma}

Denote by $\L^*_{\B_i}$ the optimal squared loss  under $\mathcal{I}(\x\to\B_i)$:
\begin{equation} \label{eq:loss:bi}
\L^*_{\B_i} = \min_{\w\in\Delta_{\kappa+1}} E_{(\x,y)\sim\D}\left[\|\w-\tilde{\bm{y}}\|^2_2|\mathcal{I}(\x\to\B_i)\right] \ ,
\end{equation}
where $\tilde{\bm{y}}\in \mathbb{R}^{\kappa+1}$ is the one-hot encoding of $y\in[\kappa+1]$. Based on Assumption~\ref{def:class-shift} and Lemma~\ref{lem:thm:ag-gini:b}, we have
\begin{theorem}\label{thm:ag-gini:b}
For $\delta\in(0,1)$ and $i\in[2t+2]\setminus[t]$, the following holds with probability at least $1-\delta$ over $S_l$ and $S_u$
\[
\left|\L^*_{\B_i} - \G_{\B_i}(S_l,S_u)\right| \leq  O\left(\frac{\kappa\sqrt{\ln(\kappa/\delta)}}{\gamma^2/(\gamma+1)}\left(\frac{1}{\sqrt{n_l}}+\frac{1}{\sqrt{n_u}}\right)\right) \ ,
\]
if  ${\mathbb{E}}_{\D}[\mu_{\B_i}(\x)]\geq \gamma$, $\mathbb{E}_{\D_\text{kc}}[\mu_{\B_i}(\x)]\geq \gamma$, $n_{\B_i,u}\geq \gamma n_u$ and $n_{\B_i,l}\geq \gamma n_l$ for some constant $\gamma\in(0,1)$.
\end{theorem}

This theorem shows that  $\G_{\B_i}(S_l,S_u)$ could be seen as an unbiased estimation of $\L^*_{\B_i}$. The detailed proof is presented in Appendix~\ref{sec:proof:thm2}, which is  similar to that of Theorem~\ref{thm:ag-gini},  but takes different analysis on the convergence of $\vartheta_{\B_i,k}(S_l,S_u)$ by considering the probabilistic assignment of neural trees.

\subsection{Deep Neural LACForest}
In the training phase, we randomly initialize encoder $h$ and neural trees $\mathcal{D}\mathcal{T}_1, \cdots, \mathcal{D}\mathcal{T}_m$, and update the encoder and  neural trees iteratively. In each iteration, we receive two mini-batches of labeled data $S'_l$ and unlabeled data $S'_u$ sampled randomly from $S_l$ and $S_u$, respectively. We introduce
\begin{equation}\label{eq:lag:dti}
\mathcal{L}_{\text{ag}}(\mathcal{D}\mathcal{T}_i)= \sum_{j\in[2t+2]\setminus[t]}\omega_{i,j}(S'_u)\G_{\B_{i,j}}(S'_l,S'_u) \ ,
\end{equation}
where $\{\B_{i,j}\}_{j\in[2t+2]\setminus[t]}$ denote  all leaf nodes of $\mathcal{D}\mathcal{T}_i$, and $\omega_{i,j}(S'_u)$ is the proportion of unlabeled data in $\B_{i,j}$, i.e.,
\[
\omega_{i,j}(S'_u) =\frac{1}{|S'_u|} \sum_{\x\in S'_u}{\mu_{\B_{i,j}}(\x)} \ .
\]
The augmented Gini loss in Eqn.~\eqref{eq:lag:dti} is a weighted average augmented Gini impurity for all leaf nodes of $\mathcal{D}\mathcal{T}_i$.

For better representation, we  introduce an auxiliary linear predictor $g\colon\mathbb{R}^{d'}\rightarrow \Delta_{\kappa}$ as done in \cite{Yan:Xie:He2021}, and take the cross entropy loss on known classes w.r.t. $S'_l$ as
\[
\mathcal{L}_{\text{ce}} = -\frac{1}{|S'_l|} \sum_{(\x,y)\in S'_l} \sum_{t\in [\kappa]} \tilde{y}_{t} \log\left(g_t\left(h\left(\x\right)\right)\right) \ .
\]
We finally get the optimization objective as follows:
\begin{equation} \label{eq:loss}
\mathcal{L} =\sum_{i\in[m]}\frac{\mathcal{L}_{\text{ag}}(\mathcal{D}\mathcal{T}_i)}{m} + \lambda_{\text{ce}} \mathcal{L}_{\text{ce}} \ ,
\end{equation}
where $\lambda_{\text{ce}}>0$  is a trade-off parameter. Intuitively, the cross entropy loss is beneficial to learning a basic representation of target data, and the augmented gini loss could improve the ability of neural trees to distinguish augmented class from known classes, as shown by Theorem~\ref{thm:ag-gini:b}.

\begin{table*}[t]

\centering
\footnotesize
\renewcommand{\arraystretch}{1.1}
\resizebox{1\linewidth}{!}{
\begin{tabular}{|ccccccccc|}
\hline
Datasets & Our LACForest & GLAC & EULAC & EVM & PAC-iForest & OSNN & LACU-SVM  & OVR-SVM\\
\hline
\textsf{segment} & .9436$\pm$.0186 & .8838$\pm$.0390$\bullet$ & .9256$\pm$.0276$\bullet$ & .8612$\pm$.0746$\bullet$ & .6678$\pm$.0763$\bullet$ & .5467$\pm$.0384$\bullet$ & .4918$\pm$.0796$\bullet$ & .6109$\pm$.0901$\bullet$   \\
\textsf{texture} & .9151$\pm$.0176 & .9130$\pm$.0328 \ \ &.9138$\pm$.0208 \ \  & .8724$\pm$.0270$\bullet$ & .7107$\pm$.0518$\bullet$ & .5956$\pm$.0405$\bullet$ & .5829$\pm$.0478$\bullet$  &  .6078$\pm$.0587$\bullet$ \\
\textsf{optdigits} & .9260$\pm$.0203 & .8842$\pm$.0384$\bullet$ & .9243$\pm$.0186 \ \ & .9069$\pm$.0234$\bullet$ & .7202$\pm$.0415$\bullet$ & .6572$\pm$.0282$\bullet$ & .7385$\pm$.0326$\bullet$ & .7547$\pm$.0552$\bullet$  \\
\textsf{satimage} & .8791$\pm$.0322 & .8215$\pm$.0526$\bullet$ & .8644$\pm$.0371$\bullet$ & .7238$\pm$.0664$\bullet$ & .7460$\pm$.0566$\bullet$ & .5101$\pm$.0380$\bullet$ & .6324$\pm$.0165$\bullet$  & .5228$\pm$.0568$\bullet$ \\
\textsf{landset} & .9243$\pm$.0206 & .8627$\pm$.0262$\bullet$ & .8857$\pm$.0236$\bullet$ & .8120$\pm$.0464$\bullet$ & .7572$\pm$.0671$\bullet$ & .5302$\pm$.0257$\bullet$ & .6289$\pm$.0124$\bullet$  & .5366$\pm$.0316$\bullet$  \\
\textsf{mfcc} & .9418$\pm$.0144 & .8937$\pm$.0194$\bullet$ & .9506$\pm$.0137$\circ$ & .8669$\pm$.0888$\bullet$ & .7901$\pm$.0710$\bullet$ & .5329$\pm$.0311$\bullet$ & .7751$\pm$.0362$\bullet$ & .6428$\pm$.0215$\bullet$ \\
\textsf{usps} & .8931$\pm$.0198 & .8645$\pm$.0382$\bullet$ & .8955$\pm$.0250 \ \ & .7959$\pm$.0697$\bullet$ & .5695$\pm$.0894$\bullet$ & .6413$\pm$.0317$\bullet$ & .7642$\pm$.0379$\bullet$  & .7488$\pm$.0464$\bullet$  \\
\textsf{har} & .9020$\pm$.0319 & .8772$\pm$.0460$\bullet$ & .8922$\pm$.0269$\bullet$ & .5016$\pm$.0558$\bullet$ & .5963$\pm$.0699$\bullet$ & .5048$\pm$.0386$\bullet$ & .5570$\pm$.0408$\bullet$ & .5046$\pm$.0251$\bullet$  \\
\textsf{mapping} & .8612$\pm$.0199 & .7964$\pm$.0350$\bullet$ & .8515$\pm$.0309$\bullet$ & .7111$\pm$.1320$\bullet$ & .6712$\pm$.1046$\bullet$ & .6218$\pm$.0605$\bullet$ & .6538$\pm$.0559$\bullet$ & .5210$\pm$.0596$\bullet$  \\
\textsf{pendigits} & .9281$\pm$.0215 & .8872$\pm$.0241$\bullet$ & .9276$\pm$.0245 \ \ & .9266$\pm$.0330 \ \  & .7709$\pm$.0742$\bullet$ & .5822$\pm$.0243$\bullet$ & .7352$\pm$.0513$\bullet$ & .6424$\pm$.0469$\bullet$  \\
\textsf{drybean} & .8932$\pm$.0208 & .9026$\pm$.0528$\circ$ & .9038$\pm$.0202$\circ$ & .7834$\pm$.0428$\bullet$ & .7555$\pm$.0693$\bullet$ & .5661$\pm$.0580$\bullet$ & .6384$\pm$.0397$\bullet$ & .5471$\pm$.0478$\bullet$  \\
\textsf{letter} & .7402$\pm$.0335 & .6331$\pm$.0394$\bullet$ & .6057$\pm$.0387$\bullet$ & .7004$\pm$.0460$\bullet$ & .4875$\pm$.0373$\bullet$ & .5980$\pm$.0231$\bullet$ &  .5547$\pm$.0538$\bullet$ & .6238$\pm$.0399$\bullet$ \\
\textsf{shuttle} & .9750$\pm$.0197 & .9444$\pm$.0506$\bullet$ &.9770$\pm$.0116 \ \  &.6245$\pm$.0347$\bullet$ & .6042$\pm$.0326$\bullet$  & .5464$\pm$.0649$\bullet$ & .6734$\pm$.0218$\bullet$ & .4862$\pm$.0249$\bullet$ \\
\textsf{drive} & .8529$\pm$.0532 & .6114$\pm$.0824$\bullet$ & .7844$\pm$.0460$\bullet$ & .7708$\pm$.0476$\bullet$ & .4296$\pm$.0578$\bullet$ & .5474$\pm$.0246$\bullet$ & .6032$\pm$.0783$\bullet$ & .5287$\pm$.1040$\bullet$  \\
\textsf{senseveh} & .7977$\pm$.0295 & .7545$\pm$.0406$\bullet$ & .7925$\pm$.0196$\bullet$ & .5685$\pm$.0572$\bullet$ & .5580$\pm$.1089$\bullet$ & .5331$\pm$.0499$\bullet$ & .5973$\pm$.0425$\bullet$ & .5726$\pm$.0688$\bullet$   \\
\textsf{mnist} & .8540$\pm$.0299 & .7947$\pm$.0408$\bullet$ & .8390$\pm$.0380$\bullet$ & .6181$\pm$.0475$\bullet$ & .5369$\pm$.0827$\bullet$ & .6504$\pm$.0386$\bullet$ & .6481$\pm$.0496$\bullet$  & .7234$\pm$.0505$\bullet$  \\
\textsf{fmnist} & .8008$\pm$.0269 & .7592$\pm$.0343$\bullet$ & .7672$\pm$.0382$\bullet$ & .6508$\pm$.0368$\bullet$ & .6090$\pm$.0577$\bullet$ & .5152$\pm$.0534$\bullet$ & .5674$\pm$.0629$\bullet$ & .5499$\pm$.0727$\bullet$ \\
\hline
average & .8840$\pm$.0591 & .8285$\pm$.0916 & .8648$\pm$.0862 & .7468$\pm$.1204 & .6459$\pm$.1042 & .5694$\pm$.0486 & .6378$\pm$.0774 & .5955$\pm$.0822  \\
\hline
\multicolumn{2}{|c}{ win/tie/loss} & \textbf{15/1/1} & \textbf{10/5/2} & \textbf{16/1/0} & \textbf{17/0/0} & \textbf{17/0/0} & \textbf{17/0/0} & \textbf{17/0/0}\\
\hline
\end{tabular}}
\caption{Experimental comparisons of accuracy (mean$\pm$std) for compared methods, and $\bullet$/$\circ$ indicates that our approach is significantly better/worse than the corresponding method (paired $t$-test at $95\%$ significance  level).}
\label{exp:1:accuracy}
\end{table*}

Notice that the optimization objective $\mathcal{L}$ in Eqn.~\eqref{eq:loss} is differentiable w.r.t. the parameters of feature encoder $h$ and neural trees $\mathcal{D}\mathcal{T}_1,\cdots,\mathcal{D}\mathcal{T}_m$, and we could update the entire model with stochastic gradient descent directly.

Given feature encoder $h$ and neural trees $\mathcal{D}\mathcal{T}_1,\cdots\mathcal{D}\mathcal{T}_m$, we could present the prediction for instance $\x\in\X$ as
\[
\mathcal{DF}(\x) =  \underset{k\in [\kappa+1]}{\arg\max} \sum_{i\in[m]} \mathcal{D}\mathcal{T}_{i,k}(\x) \  ,
\]
where $\mathcal{D}\mathcal{T}_{i,k}(\x)$ shows the probability of the $k$-th class w.r.t. neural tree $\mathcal{D}\mathcal{T}_i$, i.e., 
\[
\mathcal{D}\mathcal{T}_{i,k}(\x)=\sum_{j\in[2t]\setminus[t-1]} \mu_{\B_{i,j}}(\x)  \vartheta_{\B_{i,j},k} (S_l,S_u) \ ,
\]
with $\vartheta_{\B_{i,j},k}(S_l,S_u)$ given by Definition~\ref{def:ag-gini:B}.

Algorithm~\ref{alg:DeepLACForest} presents the detailed description of our deep neural LACForest approach. In experiments, we should take relatively large batch sizes for mini-batches $S'_l$ and $S'_u$, since $\G_{\B_{i,j}}(S'_l,S'_u)$ in Eqn.~\eqref{eq:lag:dti} converges to the optimal squared loss in the rate of $O(1/\sqrt{\min\{|S'_l|,|S'_u|\}})$ from Theorem~\ref{thm:ag-gini:b}.

\subsection{Related Works}\label{sec:relatedwork}

\citeauthor{Zhou:Chen2002} \shortcite {Zhou:Chen2002} introduced the problem of \textit{class-incremental learning} via a few labeled augmented instances.  \citeauthor{Fink:Shwartz:Singer:Ullman2006} \shortcite{Fink:Shwartz:Singer:Ullman2006} learned multiple binary classifiers for this problem and \citeauthor{Muhlbaierand:Topalis:Polikar2008}~\shortcite{Muhlbaierand:Topalis:Polikar2008} considered voting classifiers. Recent years have witnessed increasing attention on the design of practical algorithms for this problem \cite{Li:Hoiem2017,Rebuffi:Kolesnikov:Sperl:Lampert2017,Yan:Xie:He2021,Zou:Zhang:Li:li2022,Zhou:Wang:Qi:Ye:Zhan:Liu2024}. These methods can not be applied to our learning scenario directly, since we have no access to any labeled augmented instances.

Another relevant problem is \textit{open-set recognition} in computer vision.  \citeauthor{Scheirer:Rocha:Sapkota:Boult2013}~\shortcite{Scheirer:Rocha:Sapkota:Boult2013} introduced  open space risk to penalize predictions
outside the support of training data. Along this line, various approaches
have been developed based on open space risk \cite{Scheirer:Jain:Boult2014}, extreme value theory \cite{Bendale:Boult2016,Rudd:Jain:Scheirer:Boult2017}, nearest neighbors \cite{Mendes:Medeiros:Oliveira:Stein:Pazinato:Almeida2017} and generative neural networks \cite{Ge:Demyanov:Garnavi2017,Neal:Olson:Fern:Wong:Li2018,Chen:Peng:Wang:Tian2021}.
Those studies are strongly based on some geometric assumptions.

\begin{figure*}[t!]
\includegraphics[width=7in]{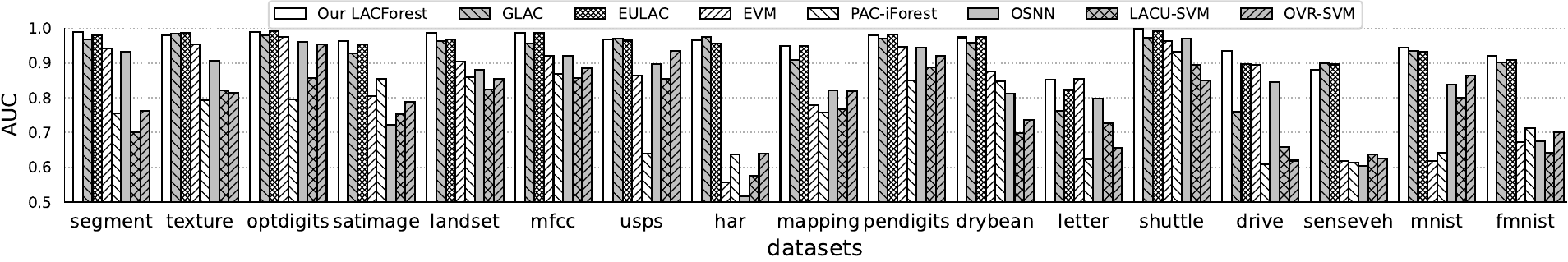}
\caption{Experimental comparisons of AUC on the detection of augmented class.}
\label{pic:comparison:auc}
\end{figure*}

\section{Experiments}\label{sec:exp}
We conduct experiments on 15 benchmark datasets and 5 image datasets, and the details are summarized in Table~\ref{tab:benchmark}. Most datasets have been well-studied in previous works on learning with augmented class.

\subsection{Evaluation of Our LACForest Approach} \label{sec:exp:lac}
For our LACForest approach, 
we compare with seven state-of-the-art approaches: GLAC \cite{Shu:He:Wang:Wei:Xian:Feng2023}, EULAC \cite{Zhang:Zhao:Ma:Zhou2020}, PAC-iForest \cite{Liu:Garrepalli:Dietterich:Fern:Hendrycks2018}, EVM \cite{Rudd:Jain:Scheirer:Boult2017}, OSNN \cite{Mendes:Medeiros:Oliveira:Stein:Pazinato:Almeida2017}, LACU-SVM \cite{Da:Yu:Zhou2014} and OVR-SVM \cite{Rifkin:Klautau2004}. More details on these approaches could be found in Appendix~\ref{sec:exp:details:stat}.

For each dataset, we randomly select half of classes as the augmented class with the rest as known classes, following  \cite{Zhang:Zhao:Ma:Zhou2020}. We then randomly sample 500 examples of known classes as labeled data $S_l$, and 1000 instances as unlabeled data $S_u$ and 100 instances as testing data. We take  $\theta=0.5$ in Eqn.~\eqref{eq:class-shift}, and more experimental settings could be found in Appendix~\ref{sec:exp:details:stat}. The performance is evaluated by 10 trials of  random selections of augmented class, and with 10 times of random data sampling. The average test accuracies are obtained over these 100 runs, as shown in Table~\ref{exp:1:accuracy}.
\begin{table*}[t!]

\centering
\footnotesize
\renewcommand{\arraystretch}{1.1}
\resizebox{1\linewidth}{!}{
\begin{tabular}{|ccccccccc|}
\hline
Datasets & Our approach & Deep-GLAC & Deep-EULAC & ARPL & G-Openmax &  OSRCI &  Openmax & Softmax-T \\
\hline
\textsf{mnist} & \textbf{.9844$\pm$.0021} & .9778$\pm$.0039 & .9596$\pm$.0033 & .9304$\pm$.0203 & .8934$\pm$.0064 & .9114$\pm$.0047& .8876$\pm$.0042 & .8834$\pm$.0029 \\
\textsf{fmnist} & \textbf{.9024$\pm$.0114} & .9010$\pm$.0148 & .8464$\pm$.0085 & .7682$\pm$.0102 & .6820$\pm$.0148 & .6912$\pm$.0070& .6672$\pm$.0139 & .5878$\pm$.0054 \\
\textsf{kuzushiji} & \textbf{.9636$\pm$.0044} & .9516$\pm$.0032 &  .8872$\pm$.0032 & .9002$\pm$.0143 & .8570$\pm$.0055 & .8602$\pm$.0046 & .8426$\pm$.0070 & .8282$\pm$.0069 \\
\textsf{svhn} &\textbf{.9238$\pm$.0134}  & .8926$\pm$.0168 & .8330$\pm$.0087  & .8060$\pm$.0069 & .7868$\pm$.0161 & .7912$\pm$.0072 & .7888$\pm$.0057 & .7252$\pm$.0047 \\
\textsf{cifar10} &\textbf{.8008$\pm$.0331} &  .7840$\pm$.0426&  .7168$\pm$.0231 &.7208$\pm$.0032 & .6560$\pm$.0125 & .6788$\pm$.0084 & .6612$\pm$.0293  & .6350$\pm$.0203 \\
\hline
average & \textbf{.9150$\pm$.0639} &  .9016$\pm$.0669 & .8486$\pm$.0793 & .8251$\pm$.0790 & .7750$\pm$.0933 & .7866$\pm$.0914 & .7695$\pm$.0915 & .7319$\pm$.1117\\
\hline
\end{tabular}}
\caption{Experimental comparisons of accuracy (mean$\pm$std) over 5 image datasets, and the best performance is highlighted in bold.}
\label{tab:exp2:accuracy}
\end{table*}
\begin{figure*}[t!]
\centering
\includegraphics[width=7in]{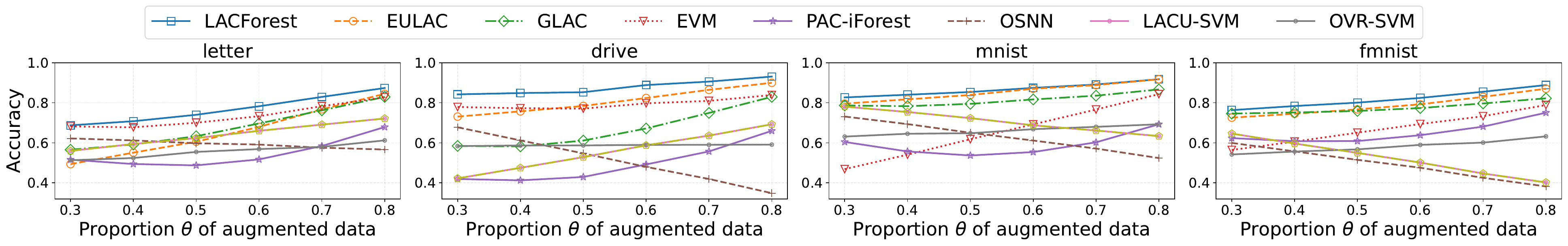}
\caption{Evaluations of LACForest over different proportions of augmented data. The larger the curve, the better the performance.}
\label{pic:comparison:theta}
\end{figure*}
\begin{figure*}[t!]
\centering
\includegraphics[width=7in]{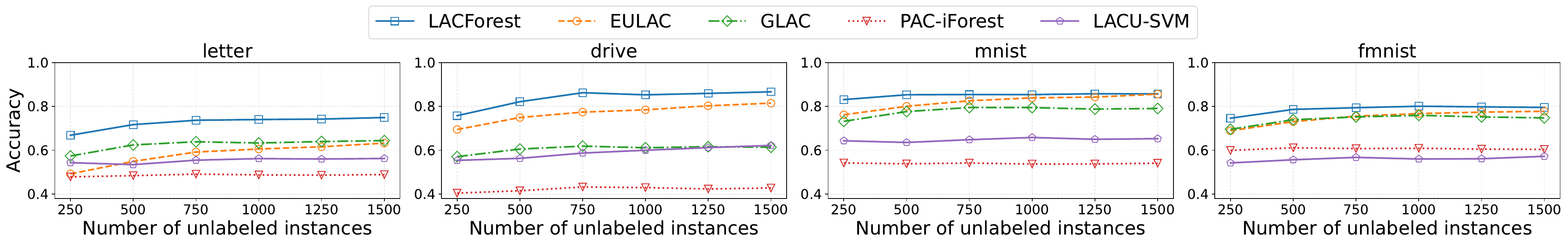}
\caption{Evaluations of LACForest under different sizes of unlabeled data $S_u$. The larger the curve, the better the performance.}\label{pic:comparison:unlabeledsize}
\end{figure*}

It is clear that our LACForest method achieves significantly better performance than previous EVM, OSNN, PAC-iForest and OVR-SVM, as our LACForest wins in most times and never loses. This is because those methods mainly focus on labeled data from known classes, but without exploring information from unlabeled data.

Our LACForest also outperforms LACU-SVM, which is heavily dependent on the low-separation assumption over data. In comparison to GLAC and EULAC, our LACForest achieves better and comparable performance  in most times, except for datasets \textsf{mfcc} and \textsf{drybean}, partially because of class imbalance in these datasets, which makes it difficult to  accurately estimate proportions of augmented class in nodes.

We also take the average AUC to show the performance on the detection of augmented class in Figure~\ref{pic:comparison:auc}. It is obvious that our LACForest takes better and comparable performance on the detection of augmented class over most datasets, since our method could
explore augmented class effectively in each leaf node, as shown by Lemma~\ref{lem:vartheta}.

\begin{figure}[t!]
\centering
\includegraphics[width=3.4in]{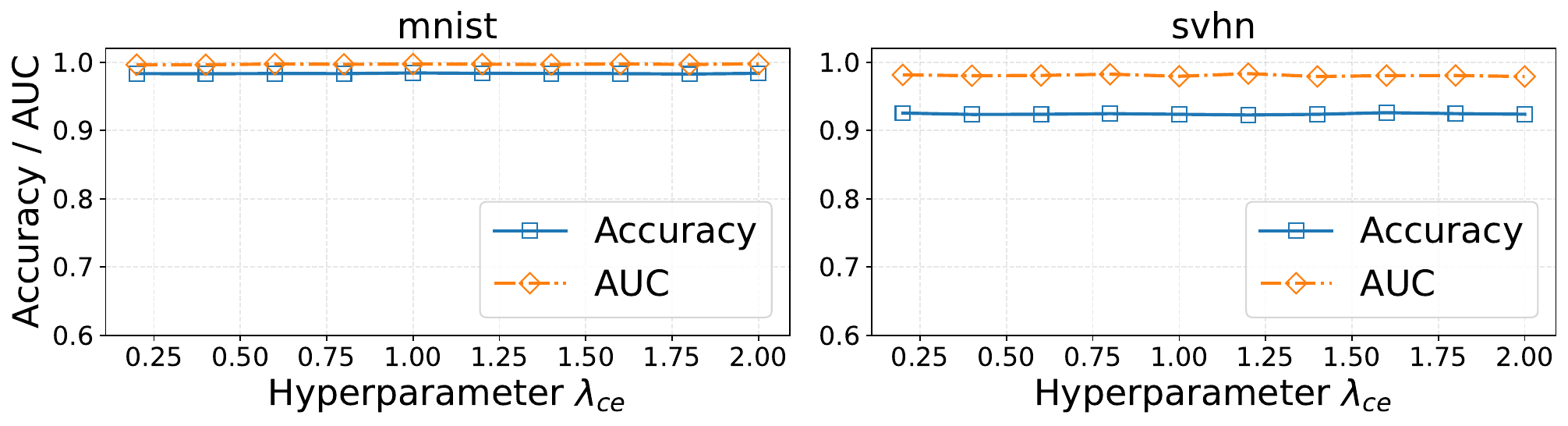}
\caption{Influence of parameter $\lambda_{\text{ce}}$.} \label{fig:lambda}
\end{figure}

\subsection{Evaluation of Our Deep Neural LACForest} \label{sec:exp:deeplac}
For our deep neural LACForest, we compare with seven deep learning methods for augmented class: {Deep-GLAC} \cite{Shu:He:Wang:Wei:Xian:Feng2023}, {ARPL} \cite{Chen:Peng:Wang:Tian2021}, {Deep-EULAC} \cite{Zhang:Zhao:Ma:Zhou2020},  {OSRCI} \cite{Neal:Olson:Fern:Wong:Li2018}, {G-Openmax} \cite{Ge:Demyanov:Garnavi2017}, {Openmax} \cite{Bendale:Boult2016} and {Softmax-T} \cite{Hendrycks:Gimpel2016}. More details on these approaches are presented in Appendix~\ref{sec:exp:details:deep}.

We take a three-layer convolutional neural network as the backbone neural network on  \textsf{mnist}, \textsf{fmnist} and \textsf{kuzushiji}, and consider VGG16 \cite{Simonyan:Zisserman2015} on \textsf{svhn} and \textsf{cifar10}, as done in
\cite{Shu:He:Wang:Wei:Xian:Feng2023}. We randomly select four classes as the augmented class and take the rest as known classes, and set $\theta=0 .4$ similarly to \cite{Shu:He:Wang:Wei:Xian:Feng2023}.

We take average test accuracies over 5 random selections of augmented class as our performance measure, and the experimental comparisons are shown in Table~\ref{tab:exp2:accuracy}. It is clear that our approach takes better performance than ARPL, G-Openmax, OSRCI, Openmax and Softmax-T, due to some additional geometric assumptions over those methods. For Deep-GLAC and Deep-EULAC, our method achieves better and comparable performance in most times, and an intuitive explanation is that our approach could effectively explore augmented class from each local region during tree partitions.

\textbf{Parameter Influence}. 
We analyze the influence of various parameters on several
datasets, and the trends are similar for other datasets.  Figure~\ref{pic:comparison:theta} shows that our LACForest gets better performance under different proportions $\theta\in[0.3,0.8]$ for augmented data. Figure~\ref{pic:comparison:unlabeledsize} shows the performance with different sizes of unlabeled data, where our LACForest achieves better and stable performance with the increase of unlabeled data, in consistency with Theorem~\ref{thm:ag-gini}. For deep neural LACForest, Figure~\ref{fig:lambda} shows that our approach is insensitive to parameter $\lambda_{\text{ce}}$ and generally works well for $\lambda_{\text{ce}}\in[0.2,2]$. Figure~\ref{fig:depth} shows the influence of the depth of neural trees, and our method takes stable results when the tree depth $l\geq 5$.

\begin{figure}[t! ]
\centering
\includegraphics[width=3.4in]{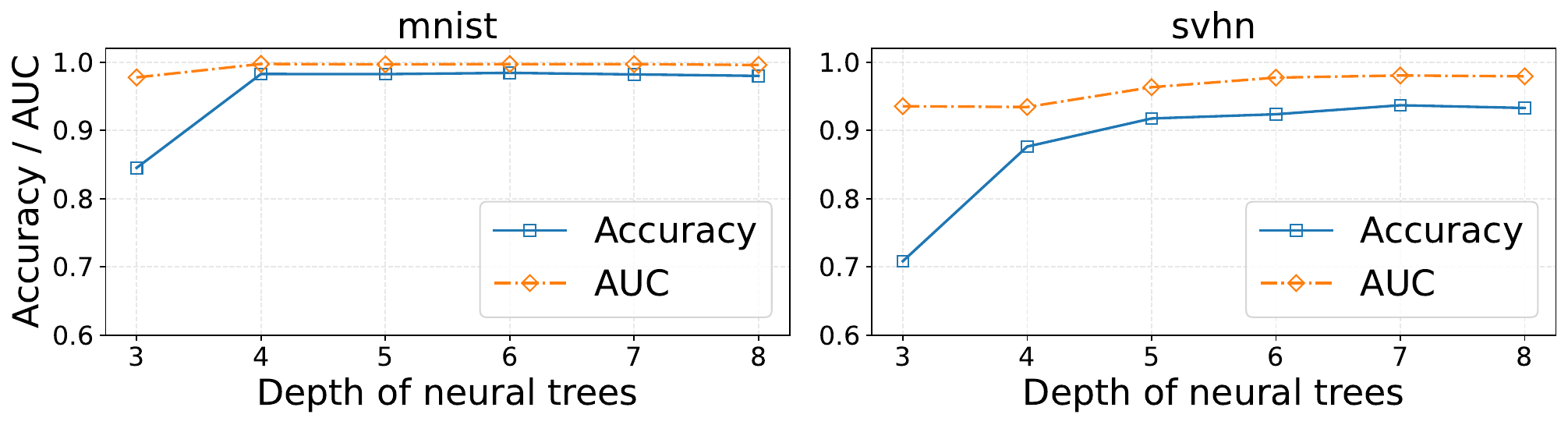}
\caption{Influence of neural trees' depth. } \label{fig:depth}
\end{figure}

\section{Conclusion} \label{sec:con}
This work studies learning with augmented class  via forests, where an augmented class may appear in testing data yet not in training data.  We introduce the augmented Gini impurity as a new splitting criterion by incorporating information of augmented class during tree construction. We develop the approach on \textit{Learning with Augmented Class via Forests}, which constructs shallow forests based on augmented Gini impurity and further splits forests with pseudo-labeled augmented instances. We also explore deep neural forests with a new optimization objective via our augmented Gini impurity. We validate the effectiveness of our methods both empirically and theoretically. An interesting future direction is to exploit our  methods under other learning settings such as streaming datasets and multiple augmented classes.

\section*{Acknowledgments}
The authors want to thank the reviewers for their helpful comments and suggestions. This research was supported by National Key R\&D Program of China (2021ZD0112802) and NSFC (62376119). Wei Gao is the corresponding author.

\bibliographystyle{named}
\bibliography{reference}

\begin{thebibliography}{}

\bibitem[\protect\citeauthoryear{Bendale and Boult}{2016}]{Bendale:Boult2016}
A.~Bendale and T.~E. Boult.
\newblock Towards open set deep networks.
\newblock In {\em Proceedings of the 29th {IEEE/CVF} International Conference
  on Computer Vision and Pattern Recognition}, pages 1563--1572, 2016.

\bibitem[\protect\citeauthoryear{Breiman}{1984}]{Breiman1984}
L.~Breiman.
\newblock {\em Classification and Regression Trees}.
\newblock CRC Press, 1984.

\bibitem[\protect\citeauthoryear{Chen \bgroup \em et al.\egroup
  }{2021}]{Chen:Peng:Wang:Tian2021}
G.~Chen, P.~Peng, X.~Wang, and Y.~Tian.
\newblock Adversarial reciprocal points learning for open set recognition.
\newblock {\em {IEEE} Transactions on Pattern Analysis and Machine
  Intelligence}, 44(11):8065--8081, 2021.

\bibitem[\protect\citeauthoryear{Costa and Pedreira}{2023}]{Costa:Pedreira2023}
V.-G. Costa and C.-E. Pedreira.
\newblock Recent advances in decision trees: An updated survey.
\newblock {\em Artificial Intelligence Review}, 56(5):4765--4800, 2023.

\bibitem[\protect\citeauthoryear{Cutler \bgroup \em et al.\egroup
  }{2007}]{cutler:Edwards:Beard:Cutler:Hess:Gibson:Lawler2007}
D.-R. Cutler, T.-C. {Edwards Jr.}, K.-H. Beard, A.~Cutler, K.-T. Hess,
  J.~Gibson, and J.-J. Lawler.
\newblock Random forests for classification in ecology.
\newblock {\em Ecology}, 88(11):2783--2792, 2007.

\bibitem[\protect\citeauthoryear{Da \bgroup \em et al.\egroup
  }{2014}]{Da:Yu:Zhou2014}
Q.~Da, Y.~Yu, and Z.-H. Zhou.
\newblock Learning with augmented class by exploiting unlabeled data.
\newblock In {\em Proceedings of the 28th {AAAI} Conference on Artificial
  Intelligence}, pages 1760--1766, 2014.

\bibitem[\protect\citeauthoryear{Fink \bgroup \em et al.\egroup
  }{2006}]{Fink:Shwartz:Singer:Ullman2006}
M.~Fink, S.~{Shalev-Shwartz}, Y.~Singer, and S.~Ullman.
\newblock Online multiclass learning by interclass hypothesis sharing.
\newblock In {\em Proceedings of the 23rd International Conference on Machine
  learning}, pages 313--320, 2006.

\bibitem[\protect\citeauthoryear{Gama \bgroup \em et al.\egroup
  }{2014}]{Gama:Zliobait:Blifet:Pechenizkiy:Bouchachia2014}
J.~Gama, I.~{\v{Z}}liobait{\.e}, A.~Bifet, M.~Pechenizkiy, and A.~Bouchachia.
\newblock A survey on concept drift adaptation.
\newblock {\em {ACM} Computing Surveys}, 46(4):1--37, 2014.

\bibitem[\protect\citeauthoryear{Ge \bgroup \em et al.\egroup
  }{2017}]{Ge:Demyanov:Garnavi2017}
Z.~Ge, S.~Demyanov, and R.~Garnavi.
\newblock Generative openmax for multi-class open set classification.
\newblock In {\em Proceedings of the 28th British Machine Vision Conference},
  pages 42.1--42.12, 2017.

\bibitem[\protect\citeauthoryear{Geng \bgroup \em et al.\egroup
  }{2021}]{Geng:Huang:Chen2021}
C.~Geng, S.-J. Huang, and S.~Chen.
\newblock Recent advances in open set recognition: A survey.
\newblock {\em IEEE Transactions on Pattern Analysis and Machine Intelligence},
  43(10):3614--3631, 2021.

\bibitem[\protect\citeauthoryear{Grinsztajn \bgroup \em et al.\egroup
  }{2022}]{Grinsztajn:Oyallon:Varoquaux2022}
L.~Grinsztajn, E.~Oyallon, and G.~Varoquaux.
\newblock Why do tree-based models still outperform deep learning on typical
  tabular data?
\newblock In {\em Advances in Neural Information Processing Systems 35}, pages
  507--520, 2022.

\bibitem[\protect\citeauthoryear{Hastie \bgroup \em et al.\egroup
  }{2009}]{Hastie:Tibshirani:Friedman2009}
T.~Hastie, R.~Tibshirani, and J.~H. Friedman.
\newblock {\em The Elements of Statistical Learning: Data Mining, Inference,
  and Prediction}.
\newblock Springer, 2009.

\bibitem[\protect\citeauthoryear{Hendrycks and
  Gimpel}{2016}]{Hendrycks:Gimpel2016}
D.~Hendrycks and K.~Gimpel.
\newblock A baseline for detecting misclassified and out-of-distribution
  examples in neural networks.
\newblock {\em CoRR/Abstract}, 2001.04295, 2016.

\bibitem[\protect\citeauthoryear{Ji \bgroup \em et al.\egroup
  }{2020}]{Ji:Wen:Zhang:Du:Wu:Zhao:Liu:Huang2020}
R.~Ji, L.~Wen, L.~Zhang, D.~Du, Y.~Wu, C.~Zhao, X.~Liu, and F.~Huang.
\newblock Attention convolutional binary neural tree for fine-grained visual
  categorization.
\newblock In {\em Proceedings of the 33rd {IEEE/CVF} Conference on Computer
  Vision and Pattern Recognition}, pages 10468--10477, 2020.

\bibitem[\protect\citeauthoryear{Kontschieder \bgroup \em et al.\egroup
  }{2015}]{Kontschieder:Fiterau:Criminisi:Bulo2015}
P.~Kontschieder, M.~Fiterau, A.~Criminisi, and S.~R. Bulo.
\newblock Deep neural decision forests.
\newblock In {\em Proceedings of the 15th {IEEE/CVF} International Conference
  on Computer Vision}, pages 1467--1475, 2015.

\bibitem[\protect\citeauthoryear{Li and Hoiem}{2017}]{Li:Hoiem2017}
Z.~Li and D.~Hoiem.
\newblock Learning without forgetting.
\newblock {\em IEEE Transactions on Pattern Analysis and Machine Intelligence},
  40(12):2935--2947, 2017.

\bibitem[\protect\citeauthoryear{Li \bgroup \em et al.\egroup
  }{2024}]{Li:Cai:Yan2024}
J.~Li, R.~Cai, and Y.~Yan.
\newblock Combinatorial routing for neural trees.
\newblock In {\em Proceedings of the 33rd International Joint Conference on
  Artificial Intelligence}, pages 4407--4415, 2024.

\bibitem[\protect\citeauthoryear{Liu \bgroup \em et al.\egroup
  }{2002}]{Liu:Lee:Yu:Li2002}
B.~Liu, W.-S. Lee, P.~S. Yu, and X.~Li.
\newblock Partially supervised classification of text documents.
\newblock In {\em Proceedings of 19th International Conference on Machine
  Learning}, pages 387--394, 2002.

\bibitem[\protect\citeauthoryear{Liu \bgroup \em et al.\egroup
  }{2018}]{Liu:Garrepalli:Dietterich:Fern:Hendrycks2018}
S.~Liu, R.~Garrepalli, T.~Dietterich, A.~Fern, and D.~Hendrycks.
\newblock Open category detection with {PAC} guarantees.
\newblock In {\em Proceedings of the 35th International Conference on Machine
  Learning}, pages 3169--3178, 2018.

\bibitem[\protect\citeauthoryear{{Mendes J{\'{u}}nior} \bgroup \em et
  al.\egroup }{2017}]{Mendes:Medeiros:Oliveira:Stein:Pazinato:Almeida2017}
P.~R. {Mendes J{\'{u}}nior}, R.~Medeiros de~Souza, R.~de~Oliveira~Werneck,
  B.~V. Stein, D.~V. Pazinato, W.~R. de~Almeida, Ot{\'{a}}vio~A.~B. Penatti,
  R.~da~Silva~Torres, and A.~Rocha.
\newblock Nearest neighbors distance ratio open-set classifier.
\newblock {\em Machine Learning}, 106(3):359--386, 2017.

\bibitem[\protect\citeauthoryear{Mu \bgroup \em et al.\egroup
  }{2017}]{Mu:Ting:Zhou2017}
X.~Mu, K.-M. Ting, and Z.-H. Zhou.
\newblock Classification under streaming emerging new classes: A solution using
  completely-random trees.
\newblock {\em IEEE Transactions on Knowledge and Data Engineering},
  29(8):1605--1618, 2017.

\bibitem[\protect\citeauthoryear{Neal \bgroup \em et al.\egroup
  }{2018}]{Neal:Olson:Fern:Wong:Li2018}
L.~Neal, M.~Olson, X.~Fern, W.-K. Wong, and F.~Li.
\newblock Open set learning with counterfactual images.
\newblock In {\em Proceedings of the 15th European Conference on Computer
  Vision}, pages 613--628, 2018.

\bibitem[\protect\citeauthoryear{Qi}{2012}]{Qi2012}
Y.~Qi.
\newblock Random forest for bioinformatics.
\newblock In {\em Ensemble Machine Learning}, pages 307--323. Springer, 2012.

\bibitem[\protect\citeauthoryear{Quinlan}{1986}]{Quinlan1986}
J.~R. Quinlan.
\newblock Induction of decision trees.
\newblock {\em Machine learning}, 1:81--106, 1986.

\bibitem[\protect\citeauthoryear{Ramaswamy \bgroup \em et al.\egroup
  }{2016}]{Ramaswamy:Scott:Tewari2016}
H.~Ramaswamy, C.~Scott, and A.~Tewari.
\newblock Mixture proportion estimation via kernel embeddings of distributions.
\newblock In {\em Proceedings of the 33rd International Conference on Machine
  Learning}, pages 2052--2060, 2016.

\bibitem[\protect\citeauthoryear{Rebuffi \bgroup \em et al.\egroup
  }{2017}]{Rebuffi:Kolesnikov:Sperl:Lampert2017}
S.-A. Rebuffi, A.~Kolesnikov, G.~Sperl, and C.~H. Lampert.
\newblock icarl: Incremental classifier and representation learning.
\newblock In {\em Proceedings of the 30th {IEEE/CVF} Conference on Computer
  Vision and Pattern Recognition}, pages 2001--2010, 2017.

\bibitem[\protect\citeauthoryear{Rifkin and Klautau}{2004}]{Rifkin:Klautau2004}
R.~Rifkin and A.~Klautau.
\newblock In defense of one-vs-all classification.
\newblock {\em The Journal of Machine Learning Research}, 5:101--141, 2004.

\bibitem[\protect\citeauthoryear{Rudd \bgroup \em et al.\egroup
  }{2017}]{Rudd:Jain:Scheirer:Boult2017}
E.~M. Rudd, L.~P. Jain, W.~J. Scheirer, and T.~E. Boult.
\newblock The extreme value machine.
\newblock {\em IEEE Transactions on Pattern Analysis and Machine Intelligence},
  40(3):762--768, 2017.

\bibitem[\protect\citeauthoryear{Scheirer \bgroup \em et al.\egroup
  }{2013}]{Scheirer:Rocha:Sapkota:Boult2013}
W.-J. Scheirer, A.~Rocha, A.~Sapkota, and T.-E. Boult.
\newblock Toward open set recognition.
\newblock {\em {IEEE} Transactions on Pattern Analysis and Machine
  Intelligence}, 35(7):1757--1772, 2013.

\bibitem[\protect\citeauthoryear{Scheirer \bgroup \em et al.\egroup
  }{2014}]{Scheirer:Jain:Boult2014}
W.-J. Scheirer, L.-P. Jain, and T.-E. Boult.
\newblock Probability models for open set recognition.
\newblock {\em {IEEE} Transactions on Pattern Analysis and Machine
  Intelligence}, 36(11):2317--2324, 2014.

\bibitem[\protect\citeauthoryear{Shu \bgroup \em et al.\egroup
  }{2023}]{Shu:He:Wang:Wei:Xian:Feng2023}
S.~Shu, S.~He, H.~Wang, H.~Wei, T.~Xian, and L.~Feng.
\newblock A generalized unbiased risk estimator for learning with augmented
  classes.
\newblock In {\em Proceedings of the 37th {AAAI} Conference on Artificial
  Intelligence}, pages 9829--9836, 2023.

\bibitem[\protect\citeauthoryear{Simonyan and
  Zisserman}{2015}]{Simonyan:Zisserman2015}
K.~Simonyan and A.~Zisserman.
\newblock Very deep convolutional networks for large-scale image recognition.
\newblock In {\em Proceedings of the 3rd International Conference on Learning
  Representations}, 2015.

\bibitem[\protect\citeauthoryear{Tanha \bgroup \em et al.\egroup
  }{2017}]{Tanha:Van:Afsarmanesh2017}
J.~Tanha, M.~{Van Someren}, and H.~Afsarmanesh.
\newblock Semi-supervised self-training for decision tree classifiers.
\newblock {\em International Journal of Machine Learning and Cybernetics},
  8:355--370, 2017.

\bibitem[\protect\citeauthoryear{Tanno \bgroup \em et al.\egroup
  }{2019}]{Tanno:Arulkumaran:Alexander:Criminisi:Nori2019}
R.~Tanno, K.~Arulkumaran, D.~Alexander, A.~Criminisi, and A.~Nori.
\newblock Adaptive neural trees.
\newblock In {\em Proceedings of the 36th International Conference on Machine
  Learning}, pages 6166--6175, 2019.

\bibitem[\protect\citeauthoryear{Topalis and
  Polikar}{2008}]{Muhlbaierand:Topalis:Polikar2008}
M.~D. Muhlbaierand~A. Topalis and R.~Polikar.
\newblock Learn++.{NC}: Combining ensemble of classifiers with dynamically
  weighted consult-and-vote for efficient incremental learning of new classes.
\newblock {\em IEEE Transactions on Neural Networks}, 20(1):152--168, 2008.

\bibitem[\protect\citeauthoryear{Wang \bgroup \em et al.\egroup
  }{2024}]{Wang:Zhang:Su:Zhu2024}
L.~Wang, X.~Zhang, H.~Su, and J.~Zhu.
\newblock A comprehensive survey of continual learning: Theory, method and
  application.
\newblock {\em IEEE Transactions on Pattern Analysis and Machine Intelligence},
  46(8):5362--5383, 2024.

\bibitem[\protect\citeauthoryear{Yan \bgroup \em et al.\egroup
  }{2021}]{Yan:Xie:He2021}
S.~Yan, J.~Xie, and X.~He.
\newblock Der: Dynamically expandable representation for class incremental
  learning.
\newblock In {\em Proceedings of the 34th {IEEE/CVF} Conference on Computer
  Vision and Pattern Recognition}, pages 3014--3023, 2021.

\bibitem[\protect\citeauthoryear{Zhang \bgroup \em et al.\egroup
  }{2020}]{Zhang:Zhao:Ma:Zhou2020}
Y.-J. Zhang, P.~Zhao, L.~Ma, and Z.-H. Zhou.
\newblock An unbiased risk estimator for learning with augmented classes.
\newblock In {\em Advances in Neural Information Processing Systems 33}, pages
  10247--10258, 2020.

\bibitem[\protect\citeauthoryear{Zhou and Chen}{2002}]{Zhou:Chen2002}
Z.-H. Zhou and Z.-Q. Chen.
\newblock Hybrid decision tree.
\newblock {\em Knowledge-Based Systems}, 15(8):515--528, 2002.

\bibitem[\protect\citeauthoryear{Zhou \bgroup \em et al.\egroup
  }{2024}]{Zhou:Wang:Qi:Ye:Zhan:Liu2024}
D.-W. Zhou, Q.-W. Wang, Z.-H. Qi, H.-J. Ye, D.-C. Zhan, and Z.~Liu.
\newblock Class-incremental learning: A survey.
\newblock {\em IEEE Transactions on Pattern Analysis and Machine Intelligence},
  46(12):9851--9873, 2024.

\bibitem[\protect\citeauthoryear{Zhou}{2022}]{Zhou2022}
Z.-H. Zhou.
\newblock Open-environment machine learning.
\newblock {\em National Science Review}, 9(8):nwac123, 2022.

\bibitem[\protect\citeauthoryear{Zou \bgroup \em et al.\egroup
  }{2022}]{Zou:Zhang:Li:li2022}
Y.~Zou, S.~Zhang, Y.~Li, and R.~Li.
\newblock Margin-based few-shot class-incremental learning with class-level
  overfitting mitigation.
\newblock In {\em Advances in Neural Information Processing Systems 35}, pages
  27267--27279, 2022.

\end{thebibliography}

\begin{appendix}
\onecolumn
\section{Detailed Theoretical Proofs} \label{sec:proof}
\subsection{Proof of Lemma~\ref{lem:vartheta}}
\begin{proof}
Recall that $S_{\C,l} = \{(\x,y)\colon (\x,y)\in S_l \text{ and } \x\in \C \}$, $S_{\C,u} =S_u\cap\C$, $n_{\C,l}=|S_{\C,l}|$ and $n_{\C,u}=|S_{\C,u}|$.
According to the Hoeffding's inequality, we have, for some $\delta_1\in(0,\frac{1}{3})$
\begin{equation} \label{lem:vartheta:eq1}
\left|\frac{n_{C,l}}{n_l} - \Pr_{(\x,y)\sim \D_\text{kc}}[\x \in C] \right|\leq \sqrt{\frac{\ln(2/\delta_1)}{2n_l}} \ ,
\end{equation}
with probability at least $1-\delta_1$. For simplicity, we introduce $n'_{\C,u}$ as follows:
\begin{equation} \label{lem:vartheta:eq:temp1}
n'_{\C,u} = \sum_{\x\in S_u} \mathbb{I}[\x\in \C] \Pr_{(\x',y')\sim \D}\left[y'\in[\kappa]|\x'=\x\right] \ .
\end{equation}
Let $f_\X(\x)$ and $f_{\X,\Y}(\x,y)$ be the probability density functions of distribution $\D_\X$ and $\D$, respectively,  and we have
\begin{eqnarray*}
&&\underset{\x\sim\D_\X}{E}\left[\mathbb{I}[\x\in\C]\Pr_{(\x',y')\sim \D}\left[y'\in[\kappa]|\x'=\x\right]\right]  \\
&&=~\int_{\x\sim\D_\X}\mathbb{I}[\x\in\C]\Pr_{(\x',y')\sim \D}\left[y'\in[\kappa]|\x'=\x\right] f_\X(\x) d\x  \\
&&=~\int_{\x\sim\D_\X}\mathbb{I}[\x\in\C]\sum_{k\in[\kappa]}f_{\X,\Y}(\x,k) d\x  \ ,
\end{eqnarray*}
and this follows that, from that $\D_\X$ is the marginal distribution of $\D$ over $\X$
\begin{eqnarray}
\nonumber&& \underset{\x\sim\D_\X}{E}\left[\mathbb{I}[\x\in\C]\Pr_{(\x',y')\sim \D}\left[y'\in[\kappa]|\x'=\x\right]\right] \\
\nonumber&&=~\int_{(\x,y)\sim\D}\mathbb{I}[y\in[\kappa],\x\in\C]f_{\X,\Y}(\x,y) d\x dy \\
&&=~\Pr_{(\x,y)\sim \D}\left[ y\in[\kappa],\x\in\C\right] \ .\label{lem:vartheta:eq:temp2}
\end{eqnarray}
From Eqns.~\eqref{lem:vartheta:eq:temp1}-\eqref{lem:vartheta:eq:temp2} and the Hoeffding's inequality, we have, for some $\delta_2\in(0,\frac{1}{3})$
\begin{eqnarray}
&&\nonumber\left|\frac{n'_{\C,u}}{n_u} - \underset{\x\sim\D_\X}{E}\left[\mathbb{I}[\x\in\C]\Pr_{(\x',y')\sim \D}\left[y'\in[\kappa]|\x'=\x\right]\right]\right| \\
&&=~\left|\frac{n'_{\C,u}}{n_u} - \Pr_{(\x,y)\sim \D}\left[y\in[\kappa],\x \in \C\right]\right|\leq \sqrt{\frac{\ln(2/\delta_2)}{2 n_u}} \ , \label{lem:vartheta:eq1.5}
\end{eqnarray}
with probability at least $1-\delta_2$. From Eqn.~\eqref{eq:class-shift}, we have
\begin{eqnarray*}
\Pr_{(\x,y)\sim \D}\left[y\in[\kappa],\x \in \C\right] = \Pr_{(\x,y)\sim \D}\left[ y\in \left[\kappa\right]\right] \Pr_{(\x,y)\sim \D}\left[\x \in \C| y\in[\kappa]\right]   =(1-\theta) \Pr_{(\x,y)\sim \D_\text{kc}}\left[\x \in \C\right] \ ,
\end{eqnarray*}
and it follows that, from Eqn.~\eqref{lem:vartheta:eq1.5}
\begin{equation}\label{lem:vartheta:eq2}
\left|\frac{n'_{\C,u}}{(1-\theta) n_u} - \Pr_{(\x,y)\sim \D_\text{kc}}[\x \in \C]\right|\leq \frac{1}{1-\theta}\sqrt{\frac{\ln(2/\delta_2)}{2n_u}} \ ,
\end{equation}
with probability at least $1-\delta_2$. Combining Eqns.~\eqref{lem:vartheta:eq1} and \eqref{lem:vartheta:eq2}, we have
\begin{equation}\label{lem:vartheta:eq3}
\left|\frac{n_{\C,l}}{n_l} - \frac{n'_{\C,u}}{(1-\theta) n_u}\right| \leq \sqrt{\frac{\ln(2/\delta_1)}{2n_l}} + \frac{1}{1-\theta}\sqrt{\frac{\ln(2/\delta_2)}{2n_u}} \ ,
\end{equation}
with probability at least $1-\delta_1-\delta_2$. From Eqn.~\eqref{lem:vartheta:eq:temp1} and the Hoeffding's inequality, we have
\begin{eqnarray}
&&\nonumber \left|\left(1-\frac{ n'_{\C,u}}{n_{\C,u}}\right) - \Pr_{(\x,y)\sim\D}[y=\kappa+1|\x\in \C] \right| \\
&&=~\left|\frac{ n'_{\C,u}}{n_{\C,u}} - \Pr_{(\x,y)\sim\D}[y\in[\kappa]|\x\in \C] \right|\leq \sqrt{\frac{\ln(2/\delta_3)}{2n_{\C,u}}} \ ,\label{lem:vartheta:eq4}
\end{eqnarray}
with probability at least $1-\delta_3$ with some $\delta_3\in(0,\frac{1}{3})$. From $n_{\C,u}\geq \gamma n_u>0$ and Definition~\ref{def:ag-gini}, we have
\[
\vartheta_{C,\kappa+1}(S_l,S_u) = \mathbb{I}[n_{\C,u}>0] \left(1-\frac{\left(1-\theta\right)n_u n_{\C,l}  }{n_l n_{\C,u} }\right)_+ =\left(1-\frac{\left(1-\theta\right)n_u n_{\C,l}  }{n_l n_{\C,u} }\right)_+\ .
\]
For simplicity, we denote by $\theta_\C=\Pr_{(\x,y)\sim\D}[y=\kappa+1|\x\in \C]$, and this follows that
\begin{eqnarray*}
&& \left|\vartheta_{\C,\kappa+1}(S_l,S_u) - \theta_\C\right| \\
&&\leq~\left|\left(1-\frac{\left(1-\theta\right) n_u n_{\C,l}}{n_ln_{\C,u}}\right) - \theta_{\C}\right|  \\
&&\leq~\left|\left(1-\frac{\left(1-\theta\right) n_u n_{\C,l}}{n_ln_{\C,u}}\right)-\left(1-\frac{ n'_{\C,u}}{n_{\C,u}}\right)\right| + \left|\left(1-\frac{ n'_{\C,u}}{n_{\C,u}}\right) - \theta_\C \right| \\
&&\leq~\left|\frac{\left(1-\theta\right) n_u n_{\C,l}}{n_ln_{\C,u}}-\frac{ n'_{\C,u}}{n_{\C,u}}\right| + \left|\left(1-\frac{ n'_{\C,u}}{n_{\C,u}}\right) - \theta_\C \right|\\
&&\leq~\frac{(1-\theta) n_u}{n_{\C,u}}\left|\frac{n_{\C,l}}{n_l}-\frac{n'_{\C,u}}{(1-\theta) n_u}\right| + \left|\left(1-\frac{ n'_{\C,u}}{n_{\C,u}}\right) - \theta_\C \right| \ ,
\end{eqnarray*}
and we further have, from  Eqns.~\eqref{lem:vartheta:eq3}-\eqref{lem:vartheta:eq4}  and $n_{\C,u}\geq \gamma n_u$
\begin{eqnarray*}
&&\left|\vartheta_{\C,\kappa+1}(S_l,S_u) - \theta_\C\right| \\
&&\leq~\frac{\left(1-\theta\right) n_u}{n_{\C,u}}(\sqrt{\frac{\ln(2/\delta_1)}{2n_l}} + \frac{1}{1-\theta}\sqrt{\frac{\ln(2/\delta_2)}{2 n_u}}) + \sqrt{\frac{\ln(2/\delta_3)}{2n_{\C,u}}} \\
&&\leq~\frac{1}{\gamma}\sqrt{\frac{\ln(2/\delta_1)}{2n_l}} + \frac{1}{\gamma}\sqrt{\frac{\ln(2/\delta_2)}{2n_u}} + \sqrt{\frac{\ln(2/\delta_3)}{2\gamma n_{u}}} \ ,
\end{eqnarray*}
with probability at least $1-\delta_1-\delta_2-\delta_3$, which completes the proof by setting $\delta_1=\delta_2=\delta_3=\delta/3$.
\end{proof}

\subsection{Proof of Theorem~\ref{thm:ag-gini}} \label{sec:proof:thm1}
We begin with some useful lemmas as follows.
\begin{lemma} \label{lem:basic}
For real numbers $a,b,c,d$ with $a,d\in[0,1]$, we have
\[
\left|ab-cd\right|\leq\left|a-c\right|+\left|b-d\right| \ .
\]
\begin{proof}
From the triangle inequality, we have
\[
\left|ab-cd\right|\leq \left|ab-ad\right| + \left|ad-cd\right| = a\left|b-d\right| + d\left|a-c\right| \ ,
\]
which completes the proof from $a,d\in[0,1]$.
\end{proof}
\end{lemma}

\begin{lemma} \label{lem:thm:ag-gini}
For any instance space $\C\subseteq \X$ and $\L^*_\C$ defined by Eqn.~\eqref{eq:r*},  we have
\[
\L^*_C = 1 - \sum_{k\in [\kappa+1]} \left(\Pr_{(\x,y)\sim\D}[y=k|\x\in \C]\right)^2\ .
\]
\end{lemma}
\begin{proof}
For any $\w\in\Delta_{\kappa+1}$, we have, by some simple calculations
\begin{eqnarray}
&&\nonumber E_{(\x,y)\sim\D}\left[\mathcal{L}(\w,\tilde{\bm{y}})|\x\in \C\right] \\ \nonumber
&&=~\sum_{k\in [\kappa+1]} \Pr_{(\x,y)\sim\D}[y=k|\x\in\C]\left(\left(1-w_{k}\right)^2 - w^2_{k} + \sum_{t\in[\kappa+1]} w^2_{t} \right) \\
\nonumber &&=~1 + \sum_{k\in [\kappa+1]} \left(w_k -\Pr_{(\x,y)\sim\D}[y=k|\x\in \C]\right)^2-\sum_{k\in [\kappa+1]} \left(\Pr_{(\x,y)\sim\D}[y=k|\x\in \C]\right)^2 \ , \\
\nonumber &&\geq~1 - \sum_{k\in [\kappa+1]} \left(\Pr_{(\x,y)\sim\D}[y=k|\x\in \C]\right)^2 \ ,
\end{eqnarray}
and this follows that, from $\L^*_C=\min\limits_{\w^*\in \Delta_{\kappa+1}}E_{(\x,y)\sim\D}\left[\mathcal{L}(\w,\tilde{\bm{y}})|\x\in \C\right] $
\begin{equation}\label{lem:thm:ag-gini:eq1}
\L^*_C\geq1 - \sum_{k\in [\kappa+1]} \left(\Pr_{(\x,y)\sim\D}[y=k|\x\in \C]\right)^2 \ .
\end{equation}
Let $w'_k=\Pr_{(\x,y)\sim\D}[y=k|\x\in C]$ for $k\in [\kappa+1]$, and we have
\[
E_{(\x,y)\sim\D}\left[\mathcal{L}(\w',\tilde{\bm{y}})|\x\in C\right] = 1 - \sum_{k\in [\kappa+1]} \left(\Pr_{(\x,y)\sim\D}[y=k|\x\in C]\right)^2 \geq \L^*_C \ ,
\]
which  completes the proof from Eqn.~\eqref{lem:thm:ag-gini:eq1}.
\end{proof}

\begin{lemma}\label{lem2:thm:ag-gini}
For instance space $\C\subseteq\X$ and $k\in[\kappa]$, if $n_{\C,u}/n_u\geq \gamma$ and $n_{\C,l}/n_l\geq \gamma$ for some constant $\gamma\in(0,1)$, then we have, with probability at least $1-\delta_1-\delta_2$ for some $\delta_1,\delta_2>0$ with $\delta_1+\delta_2<1$
\[
\left|\vartheta_{\C,k}(S_l,S_u)-\Pr_{(\x,y)\in\D}\left[y=k|\x\in \C\right]\right|\leq \sqrt{\frac{\ln(6/\delta_1)}{2\gamma n_{u}}} + \frac{1}{\gamma}\sqrt{\frac{\ln(6/\delta_1)}{2n_l}} + \frac{1}{\gamma}\sqrt{\frac{\ln(6/\delta_1)}{2n_u}} + \sqrt{\frac{\ln(2/\delta_2)}{2 \gamma n_{l}}}   \ .
\]
\end{lemma}
\begin{proof}
For simplicity, we denote by $\theta_\C=\Pr_{(\x,y)\in\D}\left[y=\kappa+1|\x\in \C\right]$.
If $\theta_\C=1$, then we have, for $k\in[\kappa]$
\begin{equation}\label{lem2:thm:ag-gini:eq0}
\Pr_{(\x,y)\in\D}\left[y=k|\x\in \C\right]=0  \ ,
\end{equation}
and this follows that, from Definition~\ref{def:ag-gini} and $\sum_{(\x,y)\in S_l}\mathbb{I}[\x\in\C\wedge y=k]=0$
\[
\vartheta_{C,k}(S_l,S_u)=(1-\vartheta_{\C,\kappa+1})\frac{\sum_{(\x,y)\in S_l} \mathbb{I}[\x\in\C\wedge y=k]}{\min(n_{\C,l},1)}=0 \ ,
\]
and we further have, from Eqn.~\eqref{lem2:thm:ag-gini:eq0}
\begin{equation} \label{lem2:thm:ag-gini:eq:temp1}
\left|\vartheta_{\C,k}(S_l,S_u)-\Pr_{(\x,y)\in\D}\left[y=k|\x\in \C\right]\right|=0 \ .
\end{equation}

On the other hand, if $\theta_\C\in[0,1)$, then we have, from Lemma~\ref{lem:vartheta}
\begin{equation} \label{lem2:thm:ag-gini:eq:temp3}
\left|\vartheta_{C,\kappa+1}(S_l,S_u)-\theta_{\C}\right|\leq\sqrt{\frac{\ln(6/\delta_1)}{2\gamma n_{u}}} + \frac{1}{\gamma}\sqrt{\frac{\ln(6/\delta_1)}{2n_l}} + \frac{1}{\gamma}\sqrt{\frac{\ln(6/\delta_1)}{2 n_u}} \ ,
\end{equation}
with probability at least $1-\delta_1$. From the Hoeffding's inequality and $n_{\C,l}\geq \gamma n_l$, we also have
\begin{equation}\label{lem2:thm:ag-gini:eq:temp2}
\left|\frac{1}{n_{\C,l}}\sum_{(\x,y)\in S_l} \mathbb{I}[\x\in\C \wedge y=k]-\Pr_{(\x,y)\sim\D_\text{kc}}[y=k|\x\in \C]\right|\leq \sqrt{\frac{\ln(2/\delta_2)}{2 n_{C,l}}}\leq \sqrt{\frac{\ln(2/\delta_2)}{2 \gamma n_{l}}} \ ,
\end{equation}
with probability at least $1-\delta_2$. Notice that $\D_{\text{kc}}$ is the marginal distribution of $\D$ over $\X\times[\kappa]$, and we have
\[
\Pr_{(\x,y)\sim\D}[y=k|\x\in \C] = \Pr_{(\x,y)\sim\D}[y\in[\kappa]|\x\in \C] \Pr_{(\x,y)\sim\D}[y=k|y\in[\kappa],\x\in \C] = (1-\theta_\C)\Pr_{(\x,y)\sim\D_\text{kc}}[y=k|\x\in \C] \ ,
\]
and this follows that, from Eqn.~\eqref{lem2:thm:ag-gini:eq:temp2} and $\theta_\C\in[0,1)$
\begin{equation} \label{lem2:thm:ag-gini:eq1}
\left|\frac{1}{n_{\C,l}}\sum_{(\x,y)\in S_l} \mathbb{I}[\x\in\C \wedge y=k]  - \frac{1}{1-\theta_{C}}\Pr_{(\x,y)\sim\D}[y=k|\x\in \C]\right|\leq \sqrt{\frac{\ln(2/\delta_2)}{2 \gamma n_{l}}}\ ,
\end{equation}
with probability at least $1-\delta_2$. From Definition~\ref{def:ag-gini} and $n_{\C,l}\geq \gamma n_l \geq 1$, we have
\[
\vartheta_{\C,k}(S_l,S_u)=\frac{1-\vartheta_{\C,\kappa+1}(S_l,S_u)}{\min(n_{\C,l},1)}\sum_{(\x,y)\in S_l} \mathbb{I}[\x\in\C\wedge y=k] = \frac{1-\vartheta_{\C,\kappa+1}(S_l,S_u)}{n_{\C,l}}\sum_{(\x,y)\in S_l} \mathbb{I}[\x\in\C\wedge y=k] \ ,
\]
and this follows that, from Lemma~\ref{lem:basic}
\begin{eqnarray*} \nonumber
&&\left|\vartheta_{\C,k}(S_l,S_u) - \Pr_{(\x,y)\sim\D}[y=k|\x\in\C]\right| \\
\nonumber
&&\leq~\left|\left(1-\vartheta_{\C,\kappa+1}(S_l,S_u)\right)\frac{\sum_{(\x,y)\in S_l} \mathbb{I}[\x\in\C\wedge y=k]}{n_{\C,l}}-(1-\theta_{\C})\frac{\Pr\limits_{(\x,y)\sim\D}[y=k|\x\in\C]}{1-\theta_{\C}}\right|  \\
&&\leq~\left|\vartheta_{\C,\kappa+1}(S_l,S_u)-\theta_\C\right| + \left|\frac{\sum_{(\x,y)\in S_l} \mathbb{I}[\x\in\C\wedge y=k]}{n_{\C,l}}-\frac{\Pr\limits_{(\x,y)\sim\D}[y=k|\x\in\C]}{1-\theta_{\C}}\right| \ ,
\end{eqnarray*}
and we further have, from Eqns.\eqref{lem2:thm:ag-gini:eq:temp3} and \eqref{lem2:thm:ag-gini:eq1}
\[
\left|\vartheta_{\C,k}(S_l,S_u) - \Pr_{(\x,y)\sim\D}[y=k|\x\in\C]\right| \leq \sqrt{\frac{\ln(6/\delta_1)}{2\gamma n_{u}}} + \frac{1}{\gamma}\sqrt{\frac{\ln(6/\delta_1)}{2n_l}} + \frac{1}{\gamma}\sqrt{\frac{\ln(6/\delta_1)}{2n_u}} + \sqrt{\frac{\ln(2/\delta_2)}{2 \gamma n_{l}}} \ .
\]
This completes the proof by combining with Eqn.~\eqref{lem2:thm:ag-gini:eq:temp1}.
\end{proof}

\noindent \textbf{Proof of Theorem~\ref{thm:ag-gini}}.
From Lemma~\ref{lem:thm:ag-gini}, we have
\[
\L^*_\C = 1 - \sum_{k\in [\kappa+1]} \left(\Pr_{(\x,y)\sim\D}[y=k|\x\in \C]\right)^2 \ ,
\]
and this follows that, from Eqn.~\eqref{eq:ag-gini}
\begin{equation} \label{thm:ag-gini:eq1}
\left|\L^*_\C - \G_\C(S_l,S_u)\right| \leq \sum_{k\in[\kappa+1]} \left|\vartheta^2_{\C,k}(S_l,S_u) - \left(\Pr_{(\x,y)\sim\D}[y=k|\x\in \C]\right)^2\right| \ .
\end{equation}
For $k=\kappa+1$, we have, from $\vartheta_{\C,\kappa+1}(S_l,S_u)\in[0,1]$ and $\theta_\C= \Pr_{(\x,y)\sim\D}[y=\kappa+1|\x\in \C]\in[0,1]$
\begin{eqnarray*}
&&\left|\vartheta^2_{\C,\kappa+1}(S_l,S_u)-\left(\Pr_{(\x,y)\sim\D}[y=\kappa+1|\x\in \C]\right)^2\right| \\
&&\leq~\left|\vartheta_{\C,\kappa+1}(S_l,S_u)-\theta_\C \right|\times \left|\vartheta_{\C,\kappa+1}(S_l,S_u) +\theta_\C \right|
\leq 2 \left|\vartheta_{\C,\kappa+1}(S_l,S_u) - \theta_\C \right| \ ,
\end{eqnarray*}
and it follows that, from Lemma~\ref{lem:vartheta}
\begin{equation}\label{thm:ag-gini:eq2}
\left|\vartheta^2_{\C,\kappa+1}(S_l,S_u)-\left(\Pr_{(\x,y)\sim\D}[y=\kappa+1|\x\in\C]\right)^2\right| \leq  2 \sqrt{\frac{\ln(6/\delta_1)}{2\gamma n_{u}}} + \frac{2}{\gamma}\sqrt{\frac{\ln(6/\delta_1)}{2n_l}} + \frac{2}{\gamma} \sqrt{\frac{\ln(6/\delta_1)}{2n_u}} \ ,
\end{equation}
with probability at least $1-\delta_1$ for some $\delta_1\in(0,1)$.
For $k\in[\kappa]$, we similarly have
\[
\left|\vartheta^2_{\C,k}(S_l,S_u)-\left(\Pr_{(\x,y)\sim\D}[y=k|\x\in \C]\right)^2\right| \leq  2 \left|\vartheta_{\C,k}(S_l,S_u)-\Pr_{(\x,y)\sim\D}[y=k|\x\in\C]\right| \ ,
\]
and this follows that, from Eqns.~\eqref{thm:ag-gini:eq1}-\eqref{thm:ag-gini:eq2} and Lemma~\ref{lem2:thm:ag-gini}
\[
\left|\L^*_\C - \G_\C(S_l,S_u)\right|\leq (2\kappa+2) \sqrt{\frac{\ln(6/\delta_1)}{2\gamma n_{u}}} + \frac{2\kappa+2}{\gamma}\sqrt{\frac{\ln(6/\delta_1)}{2n_l}} + \frac{2\kappa+2}{\gamma}\sqrt{\frac{\ln(6/\delta_1)}{2 n_u}} + 2\kappa \sqrt{\frac{\ln(2/\delta_{2})}{2 \gamma n_{l}}}\\
\]
with probability at least $1-(\kappa+1)\delta_1-\kappa\delta_2$. We finally complete the proof by setting $\delta_1=\frac{\delta}{2(\kappa+1)}$ and $\delta_2=\frac{\delta}{2\kappa}$ and some simple algebraic calculations. \qed

\subsection{Proof of Lemma~\ref{lem:thm:ag-gini:b}}
We first introduce a necessary Lemma as follows.
\begin{lemma} \label{Deep:lem1:lem1}
For $i\in[2t+2]\setminus[t]$ and $\delta\in(0,1)$, the following holds with probability at least $1-\delta$ over $S_l$ and $S_u$
\begin{eqnarray*}
\left|\Pr_{(\x,y)\sim\D}[y=\kappa+1|\mathcal{I}(\x\rightarrow \B_i)] - \vartheta_{\B_i,\kappa+1}(S_l,S_u)\right|\leq \frac{1-\theta}{\gamma} \sqrt{\frac{\ln(8/\delta)}{2n_l}} +
\frac{2}{\gamma^2}\sqrt{\frac{\ln(8/\delta)}{2n_u}}+\frac{1}{\gamma} \sqrt{\frac{\ln(8/\delta)}{2n_u}} \ ,
\end{eqnarray*}
under the class shift assumption and conditions $\mathbb{E}_{\D}[\mu_{\B}(\x)]\geq \gamma$ and $\sum_{\x\in S_u} \mu_{\B}(\x)\geq \gamma n_u$ for constant $\gamma\in(0,1)$.
\end{lemma}
\begin{proof}
For simplicity, we denote by $\eta_{k}(\x)=\Pr_{(\x',y)\in \D}[y=k|\x'=\x]$ for $k\in[\kappa+1]$ and $\x\in\X$. It is easy to observe that $\mu_{\B_i}(\x)\in(0,1)$ and $\eta_{\kappa+1}(\x)\in[0,1]$ for $\x\in\X$, and this follows that, from the Hoeffding's inequality
\begin{equation}\label{Deep:lem1:lem1:eq1}
\left|\mathbb{E}_{\D}[\mu_{\B_i}(\x)\eta_{\kappa+1}(\x)]-\frac{1}{n_u}\sum_{\x\in S_u} \mu_{\B_i}(\x)\eta_{\kappa+1}(\x)\right| \leq \sqrt{\frac{\ln(2/\delta_1)}{2n_u}} \ ,
\end{equation}
with probability at least $1-\delta_1$ for some $\delta_1\in(0,\frac{1}{4})$. We similarly have
\begin{equation}\label{Deep:lem1:lem1:eq2}
\left|\mathbb{E}_{\D}[\mu_{\B}(\x)]-\frac{1}{n_u}\sum_{\x\in S_u} \mu_{\B}(\x)\right| \leq \sqrt{\frac{\ln(2/\delta_2)}{2n_u}}\ ,
\end{equation}
with probability at least $1-\delta_2$ for some $\delta_2\in(0,\frac{1}{4})$, and it holds that
\begin{equation}\label{Deep:lem1:lem1:eq3}
\left|\mathbb{E}_{\D_\text{kc}}[\mu_{\B_i}(\x)]-\frac{1}{n_l}\sum_{(\x,y)\in S_l} \mu_{\B_i}(\x)\right| \leq \sqrt{\frac{\ln(2/\delta_3)}{2n_l}}\ ,
\end{equation}
with probability at least $1-\delta_3$ for some $\delta_3\in(0,\frac{1}{4})$. From $\sum_{k\in[\kappa]}\eta_k(\x)\in[0,1]$, we have
\begin{equation}\label{Deep:lem1:lem1:eq4}
\left|\mathbb{E}_{\D}\left[\mu_{\B_i}(\x)\sum_{k\in[\kappa]}\eta_{k}(\x)\right]-\frac{1}{n_u}\sum_{\x\in S_u} \mu_{\B_i}(\x)\sum_{k\in[\kappa]}\eta_{k}(\x)\right| \leq \sqrt{\frac{\ln(2/\delta_4)}{2n_u}}\ ,
\end{equation}
with probability at least $1-\delta_4$ for some $\delta_4\in(0,\frac{1}{4})$. From the class shift assumption in Eqn.~\eqref{eq:class-shift}, we have
\begin{eqnarray*}
&& \mathbb{E}_{\D_\text{kc}}[\mu_{\B_i}(\x)]=\mathbb{E}_\D\left[\mu_{\B_i}(\x)|y\in\left[\kappa\right]\right] \\
&&=~\frac{1}{\Pr_{\D}\left[y\in[\kappa]\right]}\mathbb{E}_{\D}\left[\mu_{\B_i}(\x)\sum_{k\in[\kappa]}\eta_{k}(\x)\right]=\frac{1}{1-\theta}\mathbb{E}_{\D}\left[\mu_{\B_i}(\x)\sum_{k\in[\kappa]}\eta_{k}(\x)\right] \ ,
\end{eqnarray*}
and we further have, from Eqn.~\eqref{Deep:lem1:lem1:eq4}
\begin{equation}\label{Deep:lem1:lem1:eq5}
\left|\mathbb{E}_{\D_\text{kc}}\left[\mu_{\B_i}(\x)\right]-\frac{1}{(1-\theta)n_u}\sum_{\x\in S_u} \mu_{\B_i}(\x)\sum_{k\in[\kappa]}\eta_{k}(\x)\right| \leq \frac{1}{1-\theta}\sqrt{\frac{\ln(2/\delta_4)}{2n_u}}\ ,
\end{equation}
with probability at least $1-\delta_4$. For simplicity, we denote by
\begin{equation} \label{eq:hat:theta}
\hat{\theta}_{\B_i}=  \frac{\sum_{\x\in S_u} \mu_{\B_i}(\x)\eta_{\kappa+1}(\x)}{\sum_{\x\in S_u} \mu_{\B_i}(\x)} \ .
\end{equation}
From $\mu_{\B_i}(\x)=\Pr[\mathcal{I}(\x\rightarrow \B_i)]$, we have
\begin{eqnarray}
\nonumber \Pr_{(\x,y)\sim\D}[y=\kappa+1|\mathcal{I}(\x\rightarrow \B_i)] =  \frac{\Pr_{\D}[y=\kappa+1,\mathcal{I}(\x\rightarrow\B_i)]}{\Pr_{\D}[\mathcal{I}(\x\rightarrow \B_i)]} =\frac{\mathbb{E}_{\D}[\mu_{\B_i}(\x)\eta_{\kappa+1}(\x)]}{\mathbb{E}_{\D}[\mu_{\B_i}(\x)]} \ ,
\end{eqnarray}
and this follows that, from Eqn.~\eqref{eq:hat:theta}
\[
\left|\Pr_{(\x,y)\sim\D}[y=\kappa+1|\mathcal{I}(\x\rightarrow \B_i)] - \hat{\theta}_{\B_i}\right| =
\left|\frac{\mathbb{E}_{\D}[\mu_{\B_i}(\x)\eta_{\kappa+1}(\x)]}{\mathbb{E}_{\D}[\mu_{\B_i}(\x)]}- \frac{\sum_{\x\in S_u} \mu_{\B_i}(\x)\eta_{\kappa+1}(\x)}{\sum_{\x\in S_u} \mu_{\B_i}(\x)}\right| \ ,
\]
and we further have, from $\mathbb{E}_{\D}[\mu_{\B_i}(\x)]\geq \gamma$ and $\frac{1}{n_u}\sum_{\x\in S_u} \mu_{\B_i}(\x)\geq \gamma$
\begin{eqnarray}
\nonumber && \left|\Pr_{(\x,y)\sim\D}[y=\kappa+1|\mathcal{I}(\x\rightarrow \B_i)] - \hat{\theta}_{\B_i}\right| \\
\nonumber && ~=~ \frac{1}{\frac{1}{n_u}\mathbb{E}_{\D}[\mu_{\B_i}(\x)] \sum_{\x\in S_u} \mu_{\B_i}(\x)} \left|\frac{1}{n_u}\mathbb{E}_{\D}[\mu_{\B_i}(\x)\eta_{\kappa+1}(\x)] \sum_{\x\in S_u} \mu_{\B_i}(\x) - \frac{1}{n_u}\mathbb{E}_{\D}[\mu_{\B_i}(\x)]\sum_{\x\in S_u} \mu_{\B_i}(\x)\eta_{\kappa+1}(\x)\right| \\
\nonumber && ~\leq~ \frac{1}{\gamma^2}  \left|\frac{1}{n_u}\mathbb{E}_{\D}[\mu_{\B_i}(\x)\eta_{\kappa+1}(\x)] \sum_{\x\in S_u} \mu_{\B_i}(\x) - \frac{1}{n_u}\mathbb{E}_{\D}[\mu_{\B_i}(\x)]\sum_{\x\in S_u} \mu_{\B_i}(\x)\eta_{\kappa+1}(\x)\right|\\
&& ~\leq~ \frac{1}{\gamma^2}\left|\mathbb{E}_{\D}[\mu_{\B_i}(\x)\eta_{\kappa+1}(\x)]-\frac{1}{n_u}\sum_{\x\in S_u} \mu_{\B_i}(\x)\eta_{\kappa+1}(\x)\right| + \frac{1}{\gamma^2}\left|\frac{1}{n_u}\sum_{\x\in S_u} \mu_{\B_i}(\x) -\mathbb{E}_{\D}[\mu_{\B_i}(\x)] \right| \ ,  \label{Deep:lem1:lem1:eq6}
\end{eqnarray}
where the last inequality holds from Lemma~\ref{lem:basic},
$\mathbb{E}_{\D}[\mu_{\B_i}(\x)\eta_{\kappa+1}(\x)]\in[0,1]$ and $\mathbb{E}_{\D}[\mu_{\B_i}(\x)]\in[0,1]$. From Eqns.~\eqref{Deep:lem1:lem1:eq1}-\eqref{Deep:lem1:lem1:eq2} and \eqref{Deep:lem1:lem1:eq6}, we could obtain that
\begin{equation}\label{Deep:lem1:lem1:eq7}
\left|\Pr_{(\x,y)\sim\D}[y=\kappa+1|\mathcal{I}(\x\rightarrow \B_i)] - \hat{\theta}_{\B_i}\right| \leq \frac{1}{\gamma^2}\left(\sqrt{\frac{\ln(2/\delta_1)}{2n_u}}+\sqrt{\frac{\ln(2/\delta_2)}{2n_u}}\right) \ ,
\end{equation}
with probability at least $1-\delta_1-\delta_2$. On the other hand, we have, from Eqn.~\eqref{eq:hat:theta} and $\sum_{k\in[\kappa+1]}\eta_k(\x)=1$ for $\x\in\X$
\[
\left|\hat{\theta}_{\B_i}-\left(1-\sum_{(\x,y)\in S_l}\frac{ (1-\theta)n_u \mu_{\B_i}(\x)} {n_l n_{\B_i,u} }\right)\right|=\left| \sum_{(\x,y)\in S_l}\frac{ (1-\theta)n_u \mu_{\B_i}(\x)} {n_l n_{\B_i,u} } -\frac{\sum_{\x\in S_u} \mu_{\B_i}(\x)\sum_{k\in[\kappa]}\eta_{k}(\x)}{\sum_{\x\in S_u} \mu_{\B_i}(\x)}\right| \ ,
\]
and this follows that,  from $\sum_{\x\in S_u} \mu_{\B}(\x)\geq\gamma n_u$ and $n_{\B_i,u}={\sum_{\x\in S_u} \mu_{\B}(\x)}$
\begin{eqnarray}
\nonumber && \left|\hat{\theta}_{\B_i}-\left(1-\sum_{(\x,y)\in S_l}\frac{ (1-\theta)n_u \mu_{\B_i}(\x)} {n_l n_{\B_i,u} }\right)\right| \\
\nonumber && ~=~ \frac{(1-\theta)n_u}{{\sum_{\x\in S_u} \mu_{\B_i}(\x)}} \left|\frac{1}{n_l}\sum_{(\x,y)\in S_l} \mu_{\B_i}(\x)   -\frac{1}{(1-\theta)n_u}\sum_{\x\in S_u} \mu_{\B_i}(\x)\sum_{k\in[\kappa]}\eta_{k}(\x)\right| \\
\nonumber && ~\leq~ \frac{1-\theta}{\gamma} \left|\frac{1}{n_l}\sum_{(\x,y)\in S_l} \mu_{\B_i}(\x) -\frac{1}{(1-\theta)n_u}\sum_{\x\in S_u} \mu_{\B_i}(\x)\sum_{k\in[\kappa]}\eta_{k}(\x)\right| \\
&& ~\leq~ \frac{1-\theta}{\gamma} \left(\left|\frac{1}{n_l}\sum_{(\x,y)\in S_l} \mu_{\B_i}(\x)-\mathbb{E}_{\D_{\text{kc}}}
[\mu_{\B_i}(\x)] \right|  + \left|\mathbb{E}_{\D_{\text{kc}}}
[\mu_{\B_i}(\x)]-\frac{1}{(1-\theta)n_u}\sum_{\x\in S_u} \mu_{\B_i}(\x)\sum_{k\in[\kappa]}\eta_{k}(\x)\right| \right) \label{Deep:lem1:lem1:eq8} \ .
\end{eqnarray}
From Eqns.~\eqref{Deep:lem1:lem1:eq3}, \eqref{Deep:lem1:lem1:eq5} and \eqref{Deep:lem1:lem1:eq8}, we have
\begin{equation}\label{Deep:lem1:lem1:eq9}
\left|\hat{\theta}_{\B_i}-\left(1-\sum_{(\x,y)\in S_l}\frac{ (1-\theta)n_u \mu_{\B_i}(\x)} {n_l n_{\B_i,u} }\right)\right|
\leq \frac{1-\theta}{\gamma} \sqrt{\frac{\ln(2/\delta_3)}{2n_l}} +\frac{1}{\gamma} \sqrt{\frac{\ln(2/\delta_4)}{2n_u}} \ ,
\end{equation}
with probability at least $1-\delta_3-\delta_4$. We finally have, from Definition~\ref{def:ag-gini:B}
\begin{eqnarray}
\nonumber &&\left|\Pr_{(\x,y)\sim\D}[y=\kappa+1|\mathcal{I}(\x\rightarrow \B_i)] - \vartheta_{\B_i,\kappa+1}(S_l,S_u)\right| \\
\nonumber  && ~\leq~ \left|\Pr_{(\x,y)\sim\D}[y=\kappa+1|\mathcal{I}(\x\rightarrow \B_i)] - \left(1-\sum_{(\x,y)\in S_l}\frac{ (1-\theta)n_u \mu_{\B_i}(\x)} {n_l n_{\B_i,u} }\right)\right| \\
\nonumber  && ~\leq~  \left|\Pr_{(\x,y)\sim\D}[y=\kappa+1|\mathcal{I}(\x\rightarrow \B_i)] - \hat{\theta}_{\B_i}\right|+ \left|\hat{\theta}_{\B_i}-\left(1-\sum_{(\x,y)\in S_l}\frac{ (1-\theta)n_u \mu_{\B_i}(\x)} {n_l n_{\B_i,u} }\right)\right| \\
\nonumber && ~\leq~ \frac{1}{\gamma^2}\sqrt{\frac{\ln(2/\delta_1)}{2n_u}}+\frac{1}{\gamma^2}\sqrt{\frac{\ln(2/\delta_2)}{2n_u}} + \frac{1-\theta}{\gamma} \sqrt{\frac{\ln(2/\delta_3)}{2n_l}} +\frac{1}{\gamma} \sqrt{\frac{\ln(2/\delta_4)}{2n_u}}  \ ,
\end{eqnarray}
with probability at least $1-\sum_{i\in[4]}\delta_i$, where the last inequality holds from Eqns.~\eqref{Deep:lem1:lem1:eq7} and \eqref{Deep:lem1:lem1:eq9}.  This completes the proof by setting $\delta_1=\delta_2=\delta_3=\delta_4=\delta/4$.
\end{proof}

\noindent \textbf{Proof of  Lemma~\ref{lem:thm:ag-gini:b}}.
From Lemma~\ref{Deep:lem1:lem1},  Lemma~\ref{lem:thm:ag-gini:b} holds for $k=\kappa+1$, and it remains to consider the case of $k\in[\kappa]$. It is easy to observe that $\mu_\B(\x)\in(0,1)$ for any $\x\in\X$, and it follows that, from the Hoeffding's inequality
\begin{equation} \label{Deep:lem1:eq1}
\left|\frac{1}{n_l}\sum\limits_{(\x,y)\in S_l}\mu_{\B_i}(\x)-\mathbb{E}_{\D_\text{kc}}\left[\mu_{\B_i}(\x)\right]\right| \leq \sqrt{\frac{\ln(2/\delta_1)}{2n_l}} \ ,
\end{equation}
with probability at least $1-\delta_1$ for some $\delta_1\in(0,\frac{1}{3})$. We similarly have, with probability at least $1-\delta_1$
\begin{equation}\label{Deep:lem1:eq2}
\left|\mathbb{E}_{\D_\text{kc}}\left[\mu_{\B_i}(\x)\mathbb{I}[y=k]\right]-\frac{1}{n_l}\sum_{(\x,y)\in S_l} \mu_{\B_i}(\x) \mathbb{I}[y=k]\right| \leq \sqrt{\frac{\ln(2/\delta_1)}{2n_l}}\ .
\end{equation}
From that $\D_\text{kc}$ is the marginal distribution of $\D$ over $\X\times[\kappa]$ and $\mu_{\B_i}(\x)=\Pr[\mathcal{I}(\x\rightarrow\B_i)]$, we have
\[
\Pr_{(\x,y)\sim\D}[y=k|y\in[\kappa],\mathcal{I}(\x\rightarrow\B_i)]=\Pr_{(\x,y)\sim\D_\text{kc}}[y=k|\mathcal{I}(\x\rightarrow\B_i)]=\frac{\mathbb{E}_{\D_\text{kc}}\left[\mu_{\B_i}(\x)\mathbb{I}[y=k]\right]}{\mathbb{E}_{\D_\text{kc}}\left[\mu_{\B_i}(\x)\right]}\ ,
\]
and it follows that, from some simple calculations
\begin{eqnarray}
\nonumber && \left|\Pr_{(\x,y)\sim\D}[y=k|y\in[\kappa],\mathcal{I}(\x\rightarrow\B_i)] -\frac{\sum_{(\x,y)\in S_l} \mu_{\B_i}(\x) \mathbb{I}[y=k]}{\sum_{(\x,y)\in S_l}  \mu_{\B_i}(\x)}\right| \\
\nonumber && ~=~ \left|\frac{\mathbb{E}_{\D_\text{kc}}\left[\mu_{\B_i}(\x)\mathbb{I}[y=k]\right]}{\mathbb{E}_{\D_\text{kc}}\left[\mu_{\B_i}(\x)\right]} -\frac{\sum_{(\x,y)\in S_l} \mu_{\B_i}(\x) \mathbb{I}[y=k]}{\sum_{(\x,y)\in S_l}  \mu_{\B_i}(\x)}\right| \\
&& ~=~\frac{\left|\mathbb{E}_{\D_\text{kc}}\left[\mu_{\B_i}(\x)\mathbb{I}[y=k]\right]\sum\limits_{(\x,y)\in S_l}\mu_{\B_i}(\x)-\mathbb{E}_{\D_\text{kc}} \left[\mu_{\B_i}(\x)\right]\sum_{(\x,y)\in S_l} \mu_{\B_i}(\x) \mathbb{I}[y=k]\right|}{\mathbb{E}_{\D_\text{kc}}\left[\mu_{\B_i}(\x)\right]\sum\limits_{(\x,y)\in S_l}\mu_{\B_i}(\x)}  \ . \label{Deep:lem1:eq3}
\end{eqnarray}
From Eqn.~\eqref{Deep:lem1:eq3}, $\mathbb{E}_{\D_\text{kc}}\left[\mu_{\B}(\x)\right]\geq\gamma$ and $\sum\limits_{(\x,y)\in S_l}\mu_{\B}(\x)\geq\gamma n_l$, we have
\begin{eqnarray}
\nonumber && \left|\Pr_{(\x,y)\sim\D}[y=k|y\in[\kappa], \mathcal{I}(\x\rightarrow\B_i)] -\frac{\sum_{(\x,y)\in S_l} \mu_{\B_i}(\x) \mathbb{I}[y=k]}{\sum_{(\x,y)\in S_l}  \mu_{\B_i}(\x)}\right| \\
\nonumber && ~\leq~ \frac{1}{\gamma^2} \left|\frac{1}{n_l}\sum\limits_{(\x,y)\in S_l}\mu_{\B_i}(\x)\mathbb{E}_{\D_\text{kc}}\left[\mu_{\B_i}(\x)\mathbb{I}[y=k]\right]-\frac{1}{n_l}\mathbb{E}_{\D_\text{kc}}\left[\mu_{\B_i}(\x)\right]\sum_{(\x,y)\in S_l} \mu_{\B_i}(\x) \mathbb{I}[y=k]\right| \\
\nonumber && ~\leq~  \frac{1}{\gamma^2} \left|\frac{1}{n_l}\sum\limits_{(\x,y)\in S_l}\mu_{\B_i}(\x)-\mathbb{E}_{\D_\text{kc}}\left[\mu_{\B_i}(\x)\right]\right| + \frac{1}{\gamma^2}\left|\mathbb{E}_{\D_\text{kc}}\left[\mu_{\B_i}(\x)\mathbb{I}[y=k]\right]-\frac{1}{n_l}\sum_{(\x,y)\in S_l} \mu_{\B_i}(\x) \mathbb{I}[y=k]\right| \ ,
\end{eqnarray}
where the last inequality holds from Lemma~\ref{lem:basic}, and we further have, from Eqns.~\eqref{Deep:lem1:eq1}-\eqref{Deep:lem1:eq2}
\begin{equation} \label{Deep:lem1:eq4}
\left|\Pr_{(\x,y)\sim\D}[y=k|y\in[\kappa],\mathcal{I}(\x\rightarrow\B_i)] -\frac{\sum_{(\x,y)\in S_l} \mu_{\B_i}(\x) \mathbb{I}[y=k]}{\sum_{(\x,y)\in S_l}\mu_{\B_i}(\x)}\right| \leq \frac{2}{\gamma^2} \sqrt{\frac{\ln(2/\delta_1)}{2n_l}} \ ,
\end{equation}
with probability at least $1-2\delta_1$. From Lemma~\ref{Deep:lem1:lem1}, we have
\begin{equation}\label{Deep:lem1:eq5}
\left| \vartheta_{\B_i,\kappa+1}(S_l,S_u)-\Pr_{(\x,y)\sim\D}[y=\kappa+1|\mathcal{I}(\x\rightarrow\B_i)]\right|\leq \frac{1-\theta}{\gamma} \sqrt{\frac{\ln(8/\delta_2)}{2n_l}} +
\frac{2}{\gamma^2}\sqrt{\frac{\ln(8/\delta_2)}{2n_u}}+\frac{1}{\gamma} \sqrt{\frac{\ln(8/\delta_2)}{2n_u}} \ ,
\end{equation}
with probability at least $1-\delta_2$ for some $\delta_2\in(0,\frac{1}{3})$. On the other hand, we also have
\begin{eqnarray}
\nonumber && \left|\Pr_{(\x,y)\sim\D}[y=k|\mathcal{I}(\x\rightarrow\B_i)] - \vartheta_{\B_i,k}(S_l,S_u)\right| \\
\nonumber &&~\leq~ \left|\Pr_{(\x,y)\sim\D}[y=k|y\in[\kappa], \mathcal{I}(\x\rightarrow\B_i)]\Pr_{(\x,y)\sim\D}[y\in[\kappa]|\mathcal{I}(\x\rightarrow\B_i)] - \frac{1-\vartheta_{\B_i,\kappa+1}(S_l,S_u)}{\sum_{(\x,y)\in S_l}\mu_{\B_i}(\x)}\sum_{(\x,y)\in S_l} \mu_{\B_i}(\x)\mathbb{I}[y=k]\right| \\
\nonumber &&~\leq~ \left|\Pr_{(\x,y)\sim\D}[y=k|y\in[\kappa],\mathcal{I}(\x\rightarrow\B_i)] -\frac{\sum_{(\x,y)\in S_l}\mu_{\B_i}(\x) \mathbb{I}[y=k]}{\sum_{(\x,y)\in S_l}\mu_{\B_i}(\x)}\right| \\
\nonumber &&~+\left|\vartheta_{\B_i,\kappa+1}(S_l,S_u)-\Pr_{(\x,y)\sim\D}[y=\kappa+1|\mathcal{I}(\x\rightarrow\B_i)]\right| \ ,
\end{eqnarray}
and it follows that, with probability at least $1-2\delta_1-\delta_2$, from Eqns.~\eqref{Deep:lem1:eq4}-\eqref{Deep:lem1:eq5}
\[
\left|\Pr_{(\x,y)\sim\D}[y=k|\mathcal{I}(\x\rightarrow\B_i)] - \vartheta_{\B_i,k}(S_l,S_u)\right| \leq \frac{2}{\gamma^2} \sqrt{\frac{\ln(2/\delta_1)}{2n_l}}+  \frac{1-\theta}{\gamma} \sqrt{\frac{\ln(8/\delta_2)}{2n_l}} +
\frac{2}{\gamma^2}\sqrt{\frac{\ln(8/\delta_2)}{2n_u}}+\frac{1}{\gamma} \sqrt{\frac{\ln(8/\delta_2)}{2n_u}} \ ,
\]
and we further have, by setting $\delta_1=\delta_2=\delta/3$
\[
\left|\Pr_{(\x,y)\sim\D}[y=k|\mathcal{I}(\x\rightarrow\B_i)] - \vartheta_{\B_i,k}(S_l,S_u)\right| \leq \frac{2}{\gamma^2} \sqrt{\frac{\ln(6/\delta)}{2n_l}}+  \frac{1-\theta}{\gamma} \sqrt{\frac{\ln(24/\delta)}{2n_l}} +
\frac{2}{\gamma^2}\sqrt{\frac{\ln(24/\delta)}{2n_u}}+\frac{1}{\gamma} \sqrt{\frac{\ln(24/\delta)}{2n_u}} \ ,
\]
with probability at least $1-\delta$, which completes the proof.
\qed

\subsection{Proof of Theorem~\ref{thm:ag-gini:b}}\label{sec:proof:thm2}
We begin with a useful lemma as follows.
\begin{lemma}\label{lem:equiv:thm:ag-gini:b}
For $i\in[2t+2]\setminus[t]$ and $\L^*_{\B_i}$ defined by Eqn.~\eqref{eq:loss:bi}, we have
\[
\L^*_{\B_i} = 1 - \sum_{k\in [\kappa+1]} \left(\Pr_{(\x,y)\sim\D}[y=k|\mathcal{I}(\x\rightarrow\B_i)]\right)^2\ .
\]
\end{lemma}
\begin{proof}
For any $\w\in\Delta_{\kappa+1}$, we have, by some simple calculations
\begin{eqnarray*}
\nonumber && E_{\D}\left[\|\w-\tilde{\bm{y}}\|^2_2|\mathcal{I}(\x\to\B_i)\right]\\ \nonumber &&=~\sum_{k\in [\kappa+1]} \Pr_{(\x,y)\sim\D}[y=k|\mathcal{I}(\x\to\B_i)]\left(\left(1-w_{k}\right)^2 - w^2_{k} + \sum_{t\in[\kappa+1]} w^2_{t} \right) \\
\nonumber &&=~1 + \sum_{k\in [\kappa+1]} \left(w_k -\Pr_{(\x,y)\sim\D}[y=k|\mathcal{I}(\x\to\B_i)]\right)^2-\sum_{k\in [\kappa+1]} \left(\Pr_{(\x,y)\sim\D}[y=k|\mathcal{I}(\x\to\B_i)]\right)^2  \\
\nonumber &&\geq~1 - \sum_{k\in [\kappa+1]} \left(\Pr_{(\x,y)\sim\D}[y=k|\mathcal{I}(\x\to\B_i)]\right)^2 \ ,
\end{eqnarray*}
and this follows that, from Eqn.~\eqref{eq:loss:bi}
\begin{equation}\label{lem:equiv:thm:ag-gini:b:eq1}
\L^*_{\B_i} = \min_{\w\in\Delta_{\kappa+1}} E_{\D}\left[\|\w-\tilde{\bm{y}}\|^2_2|\mathcal{I}(\x\to\B_i)\right] \geq 1 - \sum_{k\in [\kappa+1]} \left(\Pr_{(\x,y)\sim\D}[y=k|\mathcal{I}(\x\to\B_i)]\right)^2  \ .
\end{equation}
Let $w'_k=\Pr_{(\x,y)\sim\D}[y=k|\mathcal{I}(\x\to\B_i)]$ for $k\in [\kappa+1]$, and we have
\[
E_{\D}\left[\mathcal{L}(\w',\tilde{\bm{y}})|\mathcal{I}(\x\to\B_i)\right] = 1 - \sum_{k\in [\kappa+1]} \left(\Pr_{(\x,y)\sim\D}[y=k|\mathcal{I}(\x\to\B_i)]\right)^2 \geq \L^*_{\B_i} \ ,
\]
which  completes the proof by combining it with Eqn.~\eqref{lem:equiv:thm:ag-gini:b:eq1}.
\end{proof}

\noindent\textbf{Proof of Theorem~\ref{thm:ag-gini:b}.}
From Lemma~\ref{lem:thm:ag-gini:b} and Lemma~\ref{Deep:lem1:lem1}, and we have
\begin{equation} \label{thm:ag-gini:b:eq1}
\left|\Pr_{(\x,y)\sim\D}[y=k|\mathcal{I}(\x\rightarrow\B_i)] - \vartheta_{\B_i,k}(S_l,S_u)\right| \leq \frac{2}{\gamma^2} \sqrt{\frac{\ln(6/\delta)}{2n_l}}+  \frac{1}{\gamma} \sqrt{\frac{\ln(24/\delta)}{2n_l}} +
\frac{2}{\gamma^2}\sqrt{\frac{\ln(24/\delta)}{2n_u}}+\frac{1}{\gamma} \sqrt{\frac{\ln(24/\delta)}{2n_u}}\ ,
\end{equation}
for $k\in[\kappa+1]$ with probability at least $1-\delta_1$ for some $\delta_1\in(0,\frac{1}{\kappa+1})$.
For simplicity, we denote by
\[
\theta_{\B_i,k}=\Pr_{(\x,y)\sim\D}[y=k|\mathcal{I}(\x\rightarrow\B_i)] \ ,
\]
for $k\in[\kappa+1]$. From Eqn.~\eqref{eq:ag-gini:B}, we have
\begin{eqnarray*}
\left|\L^*_{\B_i}-\G_{\B_i}(S_l,S_u)\right| &\leq& \sum_{k\in[\kappa+1]} \left|\vartheta^2_{\B_i,k}(S_l,S_u) - \theta^2_{\B_i,k}\right| \\
&\leq& \sum_{k\in[\kappa+1]}\left|\vartheta_{\B_i,k}(S_l,S_u)+\theta_{\B_i,k}\right|\times \left|\vartheta_{\B_i,k}(S_l,S_u) - \theta_{\B_i,k}\right| \\
&\leq& \sum_{k\in[\kappa+1]} 2\left|\vartheta_{\B_i,k}(S_l,S_u) - \theta_{\B_i,k}\right| \ ,
\end{eqnarray*}
where the last inequality holds from $\vartheta_{\B_i,k}(S_l,S_u)\in[0,1]$ and $\theta_{\B_i,k}\in[0,1]$, and this follows that from Eqn.~\eqref{thm:ag-gini:b:eq1}
\begin{eqnarray*}
&& \left|\L^*_{\B_i}-\G_{\B_i}(S_l,S_u)\right| \\
&&\leq~\frac{4(\kappa+1)}{\gamma^2} \sqrt{\frac{\ln(6/\delta_1)}{2n_l}}+  \frac{2(\kappa+1)}{\gamma} \sqrt{\frac{\ln(24/\delta_1)}{2n_l}} +
\frac{4(\kappa+1)}{\gamma^2}\sqrt{\frac{\ln(24/\delta_1)}{2n_u}}+\frac{2(\kappa+1)}{\gamma} \sqrt{\frac{\ln(24/\delta_1)}{2n_u}} \ ,
\end{eqnarray*}
with probability at least $1-(\kappa+1)\delta_1$, and this completes the proof by setting $\delta_1=\delta/(\kappa+1)$. \qed

\begin{table*}[t]
\centering
\footnotesize
\renewcommand{\arraystretch}{1.1}
\resizebox{1\linewidth}{!}{
\begin{tabular}{|ccccccccc|}
\hline
Datasets & Our LACForest & GLAC & EULAC & EVM & PAC-iForest & OSNN & LACU-SVM & OVR-SVM \\
\hline
\textsf{segment} & .9380$\pm$.0196 & .8626$\pm$.0570$\bullet$ & .9102$\pm$.0384$\bullet$ & .8450$\pm$.1321$\bullet$ & .6598$\pm$.0825$\bullet$ & .6478$\pm$.0361$\bullet$ & .6654$\pm$.0763$\bullet$  & .6920$\pm$.0677$\bullet$ \\
\textsf{texture} & .9019$\pm$.0190 & .9026$\pm$.0448 \ \ & .8989$\pm$.0310 \  \ & .8682$\pm$.0267$\bullet$ & .6975$\pm$.0643$\bullet$ & .6948$\pm$.0229$\bullet$ &.7218$\pm$.0299$\bullet$  & .7072$\pm$.0383$\bullet$ \\
\textsf{optdigits} & .9269$\pm$.0266 & .8863$\pm$.0302$\bullet$ & .9291$\pm$.0248 \ \  & .9041$\pm$.0227$\bullet$ & .7231$\pm$.0461$\bullet$ & .7245$\pm$.0170$\bullet$ & .7561$\pm$.0318$\bullet$  & .8002$\pm$.0382$\bullet$ \\
\textsf{satimage} & .8478$\pm$.0415 & .7361$\pm$.1022$\bullet$ & .8249$\pm$.0552$\bullet$ & .7090$\pm$.0714$\bullet$ & .6888$\pm$.0610$\bullet$ & .5402$\pm$.0437$\bullet$ & .5960$\pm$.0247$\bullet$ & .5423$\pm$.0641$\bullet$ \\
\textsf{landset} & .9083$\pm$.0330 & .8435$\pm$.0322$\bullet$ & .8619$\pm$.0422$\bullet$ & .7921$\pm$.0520$\bullet$ & .7223$\pm$.0767$\bullet$ & .5685$\pm$.0395$\bullet$ & .6219$\pm$.0178$\bullet$ & .5741$\pm$.0493$\bullet$ \\
\textsf{mfcc} & .8806$\pm$.1015 & .7205$\pm$.1277$\bullet$ & .8832$\pm$.0892 \ \ & .8541$\pm$.1225$\bullet$ & .6186$\pm$.1372$\bullet$ & .4533$\pm$.0603$\bullet$ & .4849$\pm$.0507$\bullet$ & .6346$\pm$.0260$\bullet$ \\
\textsf{usps} & .8696$\pm$.0276 & .8549$\pm$.0339$\bullet$ & .8702$\pm$.0322 \  \ & .7800$\pm$.0654$\bullet$ & .5538$\pm$.0792$\bullet$ & .6940$\pm$.0225$\bullet$ & .7089$\pm$.0482$\bullet$ & .7788$\pm$.0343$\bullet$ \\
\textsf{har} & .8972$\pm$.0310 & .8885$\pm$.0351$\bullet$ & .8844$\pm$.0289$\bullet$ & .4526$\pm$.0701$\bullet$ & .5502$\pm$.0688$\bullet$ & .5510$\pm$.0443$\bullet$  & .4879$\pm$.0325$\bullet$ & .5426$\pm$.0343$\bullet$ \\
\textsf{mapping} & .6477$\pm$.1041 & .6021$\pm$.0953$\bullet$ & .6777$\pm$.0975$\circ$ & .5962$\pm$.1690$\bullet$ & .5083$\pm$.1406$\bullet$ & .5198$\pm$.0870$\bullet$ & .3583$\pm$.1329$\bullet$  & .4090$\pm$.1371$\bullet$ \\
\textsf{pendigits} & .9224$\pm$.0238 & .8729$\pm$.0297$\bullet$ & .9146$\pm$.0322$\bullet$ & .8834$\pm$.0347$\bullet$ & .7523$\pm$.0849$\bullet$ & .6581$\pm$.0274$\bullet$ & .7865$\pm$.0467$\bullet$ & .7114$\pm$.0339$\bullet$ \\
\textsf{drybean} & .8814$\pm$.0252 & .8821$\pm$.0742 \ \  & .8906$\pm$.0264$\circ$ & .7835$\pm$.0557$\bullet$ & .6768$\pm$.0885$\bullet$ & .6477$\pm$.0545$\bullet$ & .7397$\pm$.0416$\bullet$ & .6346$\pm$.0619$\bullet$ \\
\textsf{letter} & .6870$\pm$.0437 & .4820$\pm$.0902$\bullet$ & .5479$\pm$.0807$\bullet$ & .6728$\pm$.0340$\bullet$ & .4366$\pm$.0533$\bullet$ & .6091$\pm$.0281$\bullet$ & .4918$\pm$.0558$\bullet$ & .5134$\pm$.0658$\bullet$ \\
\textsf{shuttle} & .7765$\pm$.1906 &  .7057$\pm$.1027$\bullet$ & .8048$\pm$.1567$\circ$ & .4576$\pm$.0729$\bullet$ & .3577$\pm$.0468$\bullet$ & .3618$\pm$.0880$\bullet$ &.4217$\pm$.1264$\bullet$ & .3843$\pm$.1032$\bullet$ \\
\textsf{drive} & .8466$\pm$.0519 & .5711$\pm$.1358$\bullet$ & .7526$\pm$.0806$\bullet$ & .7660$\pm$.0434$\bullet$ & .3383$\pm$.0735$\bullet$ & .6433$\pm$.0208$\bullet$ & .3266$\pm$.0615$\bullet$ & .2832$\pm$.1169$\bullet$ \\
\textsf{senseveh} & .7962$\pm$.0254 & .7418$\pm$.0375$\bullet$ & .7752$\pm$.0194$\bullet$ & .5464$\pm$.0552$\bullet$ & .5036$\pm$.1359$\bullet$ & .5349$\pm$.0437$\bullet$ & .5726$\pm$.0398$\bullet$ & .5548$\pm$.0574$\bullet$ \\
\textsf{mnist} & .8224$\pm$.0392 & .7885$\pm$.0317$\bullet$ & .8114$\pm$.0450$\bullet$ & .4030$\pm$.0888$\bullet$ & .5282$\pm$.0840$\bullet$ & .6932$\pm$.0283$\bullet$ & .6749$\pm$.0504$\bullet$ & .7555$\pm$.0310$\bullet$ \\
\textsf{fmnist} & .7645$\pm$.0324 & .7465$\pm$.0343$\bullet$ & .7120$\pm$.0587$\bullet$ & .5524$\pm$.0723$\bullet$ & .5704$\pm$.0620$\bullet$ & .5679$\pm$.0612$\bullet$ & .6287$\pm$.0609$\bullet$ & .6039$\pm$.0709$\bullet$ \\
\hline
average & .8421$\pm$.0813 & .7699$\pm$.1212 & .8206$\pm$.0986 & .6980$\pm$.1608 & .5815$\pm$.1229 & .5947$\pm$.0934 & .5908$\pm$.1368 & .5954$\pm$.1395 \\
\hline
\multicolumn{2}{|c}{ win/tie/loss} & \textbf{16/2/0} & \textbf{10/3/4} & \textbf{17/0/0} & \textbf{17/0/0} & \textbf{17/0/0} &  \textbf{17/0/0} & \textbf{17/0/0}\\
\hline
\end{tabular}}
\caption{Experimental comparisons of macro-F1 score (mean$\pm$std) for compared methods, and $\bullet$/$\circ$ indicates that our approach is significantly better/worse than the corresponding method (paired $t$-test at $95\%$ significance  level).}
\label{exp:1:f1}
\end{table*}

\section{Experimental Details} \label{sec:exp:details}
\subsection{Evaluation of our LACForest approach} \label{sec:exp:details:stat}

\textbf{Datasets.}  All of datasets could be downloaded from \url{https://archive.ics.uci.edu/datasets} and \url{https://www.openml.org}, and we normalize the  features of raw data into the range of $[0,1]$.

\noindent \textbf{Parameter settings.}  For our LACForest,
We set the random trees number $m=100$, $\gamma=0.01$ and $\tau=\lfloor\sqrt{d}\rfloor$ on all datasets, where $d$ denotes the data dimension.
Parameters for other methods are set according to their respective references.

\noindent \textbf{Performance measure.} Following previous works \cite{Zhang:Zhao:Ma:Zhou2020,Shu:He:Wang:Wei:Xian:Feng2023}, we use three measures to evaluate the performance of compared methods, including accuracy, macro-F1 and AUC on the detection of augmented classes. For simplicity, we denote by $S_{t}=\{(\x_i,y_i)\}_{i=1}^{n_{t}}$ the testing dataset  and $f:\X\rightarrow\Y$ the learned model. We present the  performance measures in details as follows:
\begin{itemize}
\item Accuracy: the average predictive accuracy on testing data, i.e.,
\[
\text{ACC} = \frac{1}{n_t}\sum_{i=1}^{n_t} \mathbb{I} \left[f(\x_i)=y_i\right]\ .
\]
\item Macro-F1: the average F1 score on testing data, i.e.,
\[
\text{Macro-F1} = \frac{1}{\kappa+1} \sum_{k\in [\kappa+1]} \frac{2\times P_k \times R_k}{P_k+R_k} \ ,
\]
where $P_k$ and $R_k$ stand for the precision and recall of the $k$-th class, respectively.
\item AUC: the average AUC score on identifying augmented data, i.e.,
\[
\text{AUC} = \frac{1}{|S^+_{t}|\times|S^-_{t}|} \sum_{\x^+\in S^+_{t}} \sum_{\x^-\in S^-_{t}} \mathbb{I}\left[f_{\kappa+1}(\x^+) > f_{\kappa+1}(\x^-) \right] \ ,
\]
where $f_{\kappa+1}(\x)$ stands for the augmented class score w.r.t. model $f$, and $S^+_{t}$ and $S^-_{t}$ could be defined as follows:
\[
S^+_{t}=\{\x\colon (\x,y)\in S_t \wedge y=\kappa+1\} \quad \text{ and } \quad S^-_{t}=\{\x\colon (\x,y)\in S_t \wedge y\in[\kappa]\} \ .
\]
\end{itemize}

\noindent \textbf{Additional experimental results.} We present the average Macro-F1 scores and AUC scores over 100 different configurations for each dataset in Table~\ref{exp:1:f1} and Table~\ref{exp:1:auc}, respectively.
As can be seen, our approach exhibits better performance than other compared methods, since the win/tie/loss count clearly indicates that our approach wins in most times.

\begin{table*}[t]
\centering
\footnotesize
\renewcommand{\arraystretch}{1.1}
\resizebox{1\linewidth}{!}{
\begin{tabular}{|ccccccccc|}
\hline
Datasets & Our LACForest & GLAC & EULAC & EVM & PAC-iForest & OSNN & LACU-SVM & OVR-SVM \\
\hline
\textsf{segment} & .9891$\pm$.0082 & .9687$\pm$.0141$\bullet$ & .9807$\pm$.0141$\bullet$ & .9423$\pm$.0309$\bullet$ & .7547$\pm$.1005$\bullet$ & .9326$\pm$.0447$\bullet$ & .7019$\pm$.0254$\bullet$ & .7631$\pm$.1040$\bullet$ \\
\textsf{texture} & .9799$\pm$.0083 & .9842$\pm$.0088$\circ$ & .9875$\pm$.0084$\circ$ & .9546$\pm$.0179$\bullet$ & .7925$\pm$.0460$\bullet$ & .9054$\pm$.0292$\bullet$ & .8207$\pm$.0588$\bullet$ & .8133$\pm$.0904$\bullet$ \\
\textsf{optdigits} & .9894$\pm$.0072 & .9798$\pm$.0075$\bullet$  & .9916$\pm$.0054 \ \ & .9741$\pm$.0101$\bullet$ & .7948$\pm$.0484$\bullet$ & .9609$\pm$.0142$\bullet$ &.8567$\pm$.0312$\bullet$ & .9530$\pm$.0227$\bullet$ \\
\textsf{satimage} & .9642$\pm$.0209 & .9281$\pm$.0504$\bullet$ & .9542$\pm$.0264$\bullet$ & .8051$\pm$.0652$\bullet$ & .8551$\pm$.0654$\bullet$ & .7230$\pm$.1056$\bullet$ & .7524$\pm$.0564$\bullet$ & .7891$\pm$.0700$\bullet$ \\
\textsf{landset} & .9869$\pm$.0054 & .9625$\pm$.0107$\bullet$ & .9685$\pm$.0120$\bullet$ & .9039$\pm$.0380$\bullet$ & .8599$\pm$.0675$\bullet$ & .8803$\pm$.0445$\bullet$ & .8231$\pm$.0355$\bullet$ & .8454$\pm$.0425$\bullet$ \\
\textsf{mfcc} & .9877$\pm$.0044 & .9569$\pm$.0141$\bullet$ & .9871$\pm$.0066 \  \ & .9215$\pm$.0589$\bullet$ & .8691$\pm$.0476$\bullet$ & .9197$\pm$.0598$\bullet$ & .8572$\pm$.0393$\bullet$ & .8862$\pm$.0394$\bullet$ \\
\textsf{usps} & .9687$\pm$.0114 & .9696$\pm$.0091 \ \  & .9643$\pm$.0267$\bullet$   & .8640$\pm$.0700$\bullet$ & .6401$\pm$.1106$\bullet$ & .8980$\pm$.0475$\bullet$ & .8547$\pm$.0498$\bullet$ & .9346$\pm$.0161$\bullet$ \\
\textsf{har} & .9662$\pm$.0210 & .9740$\pm$.0146$\circ$ & .9558$\pm$.0290$\bullet$ & .5564$\pm$.0493$\bullet$ & .6375$\pm$.0956$\bullet$ & .5169$\pm$.0824$\bullet$ & .5752$\pm$.0628$\bullet$ & .6400$\pm$.1333$\bullet$ \\
\textsf{mapping} & .9482$\pm$.0157 & .9093$\pm$.0322$\bullet$ & .9486$\pm$.0154 \ \ & .7786$\pm$.1671$\bullet$ & .7564$\pm$.1336$\bullet$ & .8202$\pm$.0897$\bullet$ & .7659$\pm$.0894$\bullet$ & .8192$\pm$.0771$\bullet$ \\
\textsf{pendigits} & .9796$\pm$.0105 & .9692$\pm$.0115$\bullet$ & .9820$\pm$.0211 \ \ & .9475$\pm$.0236$\bullet$ & .8498$\pm$.0694$\bullet$ & .9437$\pm$.0285$\bullet$ & .8872$\pm$.0436$\bullet$ & .9209$\pm$.0254$\bullet$ \\
\textsf{drybean} & .9738$\pm$.0103 & .9583$\pm$.0266$\bullet$ & .9747$\pm$.0121 \ \ & .8757$\pm$.0404$\bullet$ & .8484$\pm$.0749$\bullet$ & .8117$\pm$.0621$\bullet$ & .6981$\pm$.0418$\bullet$ & .7360$\pm$.1092$\bullet$ \\
\textsf{letter} & .8519$\pm$.0268 & .7626$\pm$.0410$\bullet$ & .8225$\pm$.0468$\bullet$ & .8539$\pm$.0174 \ \  & .6238$\pm$.0571$\bullet$ & .7977$\pm$.0253$\bullet$ & .7269$\pm$.0408$\bullet$ & .6561$\pm$.0597$\bullet$ \\
\textsf{shuttle}& .9976$\pm$.0024 & .9734$\pm$.0284$\bullet$ & .9915$\pm$.0067$\bullet$ & .9642$\pm$.0039$\bullet$ & .9318$\pm$.0074$\bullet$ &.9701$\pm$.0193$\bullet$ & .8952$\pm$.0047$\bullet$ & .8503$\pm$.0029$\bullet$\\
\textsf{drive} & .9340$\pm$.0444 & .7589$\pm$.0975$\bullet$ & .8973$\pm$.0442$\bullet$ & .8955$\pm$.0396$\bullet$ & .6076$\pm$.0861$\bullet$ & .8448$\pm$.0368$\bullet$ & .6588$\pm$.0626$\bullet$ & .5992$\pm$.1241$\bullet$ \\
\textsf{senseveh} & .8814$\pm$.0293 & .8986$\pm$.0256$\circ$ & .8958$\pm$.0268$\circ$ & .6174$\pm$.0712$\bullet$ & .6138$\pm$.1410$\bullet$ & .6038$\pm$.0680$\bullet$ & .6376$\pm$.0837$\bullet$ & .6249$\pm$.0943$\bullet$ \\
\textsf{mnist} & .9435$\pm$.0292 & .9341$\pm$.0248$\bullet$ & .9330$\pm$.0330$\bullet$ & .6189$\pm$.0477$\bullet$ & .6414$\pm$.1197$\bullet$ & .8370$\pm$.0297$\bullet$ & .7989$\pm$.0456$\bullet$ & .8632$\pm$.0337$\bullet$ \\
\textsf{fmnist} & .9216$\pm$.0299 & .9020$\pm$.0379$\bullet$ & .9097$\pm$.0330$\bullet$ & .6725$\pm$.0409$\bullet$ & .7132$\pm$.0942$\bullet$ & .6755$\pm$.0658$\bullet$ & .6425$\pm$.0531$\bullet$ & .7016$\pm$.1201$\bullet$ \\
\hline
average & .9567$\pm$.0391 & .9288$\pm$.0668 & .9497$\pm$.0443 & .8321$\pm$.1318 & .7521$\pm$.1045 & .8260$\pm$.1253 & .7619$\pm$.0948 & .7880$\pm$.1092\\
\hline
\multicolumn{2}{|c}{ win/tie/loss} & \textbf{13/1/3} & \textbf{10/5/2} & \textbf{16/1/0} & \textbf{17/0/0} & \textbf{17/0/0} & \textbf{17/0/0} & \textbf{17/0/0}\\
\hline
\end{tabular}}
\caption{Experimental comparisons of AUC (mean$\pm$std) on the detection of augmented class, and $\bullet$/$\circ$ indicates that our approach is significantly better/worse than the corresponding method (paired $t$-test at $95\%$ significance  level).}
\label{exp:1:auc}
\end{table*}

\subsection{Evaluation of our deep neural LACForest} \label{sec:exp:details:deep}
\textbf{Datasets.} We follow the data setting of each dataset in previous studies \cite{Zhang:Zhao:Ma:Zhou2020,Shu:He:Wang:Wei:Xian:Feng2023}. All the datasets contain ten classes, among which we randomly select six classes as known classes and four other classes as the augmented class. The detailed description of each dataset is presented  as follows.
\begin{itemize}
\item \textsf{Mnist}  contains 28$\times$28 monochrome images of digital numbers. The labeled data contains instances of six known classes with 4000 instances per class. On the other hand,
the unlabeled data contains unlabeled instances from 10 classes with 1000 instances in each class. Thus, the total number of training instances is 34000. The testing set consists
of ten classes with 100 instances in each class, and thus the total number  is
1000. The dataset can be downloaded from \url{http://yann.lecun.com/exdb/mnist/}.
\item \textsf{Fmnist} includes 28$\times$28 monochrome images of fashion apparel,
and it takes similar data setting to that of \textsf{mnist}.  The dataset can be downloaded from \url{https://github.com/zalandoresearch/fashion-mnist}.
\item \textsf{Kuzushiji} contains 28$\times$28 monochrome images of handwritten japanese characters, and it takes similar data setting to that of \textsf{mnist}. The dataset can be downloaded from \url{https://github.com/rois-codh/kmnist}.
\item \textsf{Svhn} is short for Street View House Numbers, which includes 32$\times$32 colored image
samples of numbers. The training set consists of six
known classes with 4500 subsampled instances in each class and ten unlabeled classes with
3000 instances in each class, and the total number of instances is 57000. The
testing set contains instances of all ten classes with 100 instances in each class. SVHN
dataset can be downloaded from \url{http://ufldl.stanford.edu/housenumbers/}.
\item \textsf{Cifar10} contains natural images of size 32$\times$32. The training set contains
six known classes with 5000 subsampled instances in each class and unlabeled instances
from all ten classes with 900 subsampled instances in each class; thus, the total number
of training instance is 39000. The testing set contains all ten classes with 100
instances per class. Cifar-10 dataset can be downloaded from \url{https://www.cs.toronto.edu/~kriz/cifar.html}.
\end{itemize}

\noindent\textbf{Compared methods.}
This section introduces the details of compared approaches, where Deep-GLAC \cite{Shu:He:Wang:Wei:Xian:Feng2023} and Deep-EULAC \cite{Zhang:Zhao:Ma:Zhou2020} could be seen as direct extensions of GLAC and EULAC in section~\ref{sec:exp:lac}, respectively.

\begin{itemize}
\item Softmax-T \cite{Hendrycks:Gimpel2016} trains $\kappa$ classifiers $f_1(\x),\cdots,f_\kappa(\x)$ to approximate the posterior probabilities $\Pr[y=k|\x]$ for $k\in[\kappa+1]$. An instance $\x$ is predicted to be the augmented class when the maximum outputs $\max_{k\in[\kappa]}f_k(\x)$ is less than a given threshold.
\item Openmax \cite{Bendale:Boult2016} can be seen as a calibrated version of Softmax-T, where replaces the Softmax layer with a new Openmax layer based on Weibull calibration.
\item G-Openmax \cite{Ge:Demyanov:Garnavi2017} is an extension of Openmax approach with generative neural networks.
\item OSRCI \cite{Neal:Olson:Fern:Wong:Li2018} is short for Open-set Recognition using Counterfactual Images, which
utilizes a counterfactual image generation technique to train an additional classifier for the augmented class.
\item ARPL \cite{Chen:Peng:Wang:Tian2021} is short for Adversarial Reciprocal Points Learning, which introduces a novel framework for open set recognition by leveraging reciprocal points to model latent open space.

\end{itemize}

\noindent\textbf{Parameter settings.}  For our approach, we construct 3 neural trees of depth $l=6$, and set the hyper-parameter $\lambda_{\text{ce}}=1$ for all datasets. We set the number of epochs as 500 for all datasets, and take a relatively large batch size 512 for both labeled and unlabeled data according to Theorem~\ref{thm:ag-gini:b}. The loss is optimized by SGD optimizer with weight decay $5\times 10^{-3}$. The initial learning rate is set to $10^{-2}$ and is progressively reduced to $10^{-3}$ through a cosine annealing scheduler.
For the estimation of proportion $\theta$, we randomly sample 1000 labeled examples and 1000 unlabeled instances for ten times, and take the average estimated proportion by previous method \cite{Ramaswamy:Scott:Tewari2016}.
For other compared methods, we select their  parameters according to their corresponding references.

\noindent\textbf{Additional experimental results.} We present the average Macro-F1 scores and AUC scores for each dataset in Table~\ref{tab:exp2:f1} and Table~\ref{tab:exp2:auc}, respectively. It is observable that our approach exhibits better performance than other compared methods in most times, which indicates the effectiveness of our approach under different performance measures.

\begin{table*}[t]
\centering
\footnotesize
\renewcommand{\arraystretch}{1.1}
\resizebox{1\linewidth}{!}{
\begin{tabular}{|ccccccccc|}
\hline
Datasets & Our approach & Deep-GLAC & Deep-EULAC & ARPL & G-Openmax &  OSRCI & Openmax & Softmax-T \\
\hline
\textsf{mnist} & \textbf{.9833$\pm$.0027} & .9778$\pm$.0035 & .9601$\pm$.0027 & .9390$\pm$.0171 & .8995$\pm$.0053 & .9171$\pm$.0044& .8972$\pm$.0041 & .8912$\pm$.0030\\
\textsf{fmnist} & \textbf{.8986$\pm$.0141} & .8945$\pm$.0171 & .8436$\pm$.0076 & .7814$\pm$.0092 & .6925$\pm$.0135 & .7099$\pm$.0084& .6832$\pm$.0178 & .6297$\pm$.0102  \\
\textsf{kuzushiji} & \textbf{.9629$\pm$.0064} & .8867$\pm$.0038 & .9504$\pm$.0041 & .9117$\pm$.0120 & .8582$\pm$.0053 & .8611$\pm$.0038& .8503$\pm$.0043 & .8333$\pm$.0060   \\
\textsf{svhn} &\textbf{.9237$\pm$.0140}  & .8936$\pm$.0145 & .8236$\pm$.0094  &.8050$\pm$.0080 & .7871$\pm$.0157 & .8039$\pm$.0206 & .8026$\pm$.0066 & .7642$\pm$.0034 \\
\textsf{cifar10} &\textbf{.7936$\pm$.0299} & .7911$\pm$.0429 & .7119$\pm$.0231 & .7271$\pm$.0031  & .6777$\pm$.0205  & .7061$\pm$.0133  & .6957$\pm$.0391  &  .6857$\pm$.0264   \\
\hline
average &\textbf{.9124$\pm$.0663} & .8887$\pm$.0592 & .8579$\pm$.0913 & .8328$\pm$.0801  & .7830$\pm$.0878  & .7996$\pm$.0829 & .7858$\pm$.0843  &  .7608$\pm$.0950    \\
\hline
\end{tabular}}
\caption{Experimental comparisons of macro-F1 score (mean$\pm$std) over 5 image datasets, and  the best performance is highlighted in bold.}\label{tab:exp2:f1}
\end{table*}

\begin{table*}[!t]
\centering
\tiny
\footnotesize
\renewcommand{\arraystretch}{1.1}
\resizebox{1\linewidth}{!}{
\begin{tabular}{|ccccccccc|}
\hline
Datasets & Our approach & Deep-GLAC & Deep-EULAC & ARPL & G-Openmax &  OSRCI & Openmax & Softmax-T \\
\hline
\textsf{mnist} & \textbf{.9983$\pm$.0019} & .9974$\pm$.0010 & .9927$\pm$.0009 & .9517$\pm$.0452 & .9396$\pm$.0031 & .9407$\pm$.0016 & .9399$\pm$.0032 & .9367$\pm$.0029 \\
\textsf{fmnist} & \textbf{.9784$\pm$.0061} & .9751$\pm$.0149 & .9505$\pm$.0080 & .8607$\pm$.0046 & .8423$\pm$.0042 & .8481$\pm$.0049& .8397$\pm$.0070 & .8292$\pm$.0042  \\
\textsf{kuzushiji} & \textbf{.9939$\pm$.0026} & .9927$\pm$.0035 & .9805$\pm$.0019 & .9581$\pm$.0078 & .9090$\pm$.0029 & .9110$\pm$.0043& .9066$\pm$.0030 & .8984$\pm$.0028    \\
\textsf{svhn} &.9795$\pm$.0020  & \textbf{.9839$\pm$.0051} & .9444$\pm$.0073  & .9195$\pm$.0026& .8678$\pm$.0059 & .8969$\pm$.0029 & .8654$\pm$.0055 & .8535$\pm$.0063 \\
\textsf{cifar10} &\textbf{.9292$\pm$.0095} & .9254$\pm$.0071 &  .8535$\pm$.0153 & .8182$\pm$.0051 & .7574$\pm$.0052 & .7947$\pm$.0062  & .7740$\pm$.0143 & .7292$\pm$.0188   \\
\hline
average &\textbf{.9759$\pm$.0246} & .9749$\pm$.0259 & .9443$\pm$.0489 & .9016$\pm$.0541  & .8632$\pm$.0626  & .8783$\pm$.0514 & .8651$\pm$.0570  &  .8494$\pm$.0706  \\
\hline
\end{tabular}}
\caption{Experimental comparisons of AUC (mean$\pm$std) on the detection of augmented class over 5 image datasets. Notice that the best performance is highlighted in bold.}\label{tab:exp2:auc}
\end{table*}


\end{appendix}

\end{document}